\pdfoutput=1 

\documentclass[sigconf]{aamas} 

\usepackage{balance} 

\setcopyright{ifaamas}
\acmConference[AAMAS '21]{Proc.\@ of the 20th International Conference on Autonomous Agents and Multiagent Systems (AAMAS 2021)}{May 3--7, 2021}{Online}{U.~Endriss, A.~Now\'{e}, F.~Dignum, A.~Lomuscio (eds.)}
\copyrightyear{2021}
\acmYear{2021}
\acmDOI{}
\acmPrice{}
\acmISBN{}

\acmSubmissionID{67}

\title{Aggregating Bipolar Opinions (With Appendix)}

\author{Stefan Lauren}
\affiliation{
  \institution{Imperial College London, UK}}
\email{s18@imperial.ac.uk}

\author{Francesco Belardinelli}
\affiliation{
  \institution{Imperial College London, UK}}
\affiliation{
  \institution{Universite d'Evry, France}}
\email{francesco.belardinelli@imperial.ac.uk}

\author{Francesca Toni}
\affiliation{
  \institution{Imperial College London,UK}}
\email{f.toni@imperial.ac.uk}

\begin{abstract}
	We introduce a novel method to aggregate Bipolar Argumentation \FT{(BA)} 
	Frameworks expressing opinions by different parties in debates. We use Bipolar Assumption-based Argumentation (ABA) 
	as an all-encompassing formalism for 
	\FT{BA} 
	under different semantics. By leveraging on recent results on judgement aggregation in Social Choice Theory, we prove several preservation results, both positive and negative, for relevant properties of Bipolar ABA.
\end{abstract}


\newcommand{\BibTeX}{\rm B\kern-.05em{\sc i\kern-.025em b}\kern-.08em\TeX}
\newtheorem{corollary}{Corollary}
\newtheorem{lemma}{Lemma}
\newtheorem{definition}{Definition}
\newtheorem{example}{Example}
\DeclareFontFamily{U}{mathx}{\hyphenchar\font45}
\DeclareFontShape{U}{mathx}{m}{n}{ <-> mathx10 }{}
\DeclareSymbolFont{mathx}{U}{mathx}{m}{n}
\DeclareFontSubstitution{U}{mathx}{m}{n}
\DeclareMathAccent{\widebar}{\mathalpha}{mathx}{"73}
\newcommand{\contraries}{\widebar{\hspace{0.35cm}}}
\newcommand{\contrary}[1]{\widebar{#1}}
\DeclareMathAlphabet{\mathbfcal}{OMS}{cmsy}{b}{n}

\newcommand{\FT}[1]{\textcolor{black}{#1}}

\begin{document}


\pagestyle{fancy}
\fancyhead{}


\maketitle 


\section{Introduction}
There is a long and well-established tradition in knowledge representation and reasoning to formally describe debates as exchanges of opinions by the parties involved through attacks \cite{Ulle2017} and supports \cite{Rago2017} between arguments understood abstractly as in {\em Abstract Argumentation (AA)}~\cite{Dung1995} 
\FT{or} as in \textit{Bipolar Argumentation} (BA) \cite{CayrolL05a, CayrolL13}.
%
\FT{When these debates emerge in multi-agent systems~\cite{Ganzer2019, Rago2017},}
a key question concerns opinion aggregation, namely how we can obtain a collective consensus from several opinions expressed as argumentation frameworks, in such a way that the agents' opinions are well-portrayed in the collective outcome \cite{Tohme2017}. Recently, Chen and Endriss \cite{Ulle2017, Ulle2018} applied aggregation procedures from Social Choice Theory \cite{Christian2013} to AA Frameworks (AAFs), by leveraging especially on judgement aggregation \cite{Grossi2014}, and  proving the preservation of interesting properties of AAFs, such as conflict-freeness, acyclicity 
{and 
\FT{extensions} according to several semantics}.
Some efforts had been made to provide procedures for the aggregation of \emph{Quantitative Argumentation Debate Frameworks} \cite{Rago2017, BaroniRTAB15, RagoTAB16} (a form of \emph{BA Frameworks (BAFs)} incorporating both attacks and supports but equipped with gradual, rather than extension-based, semantics), 
but the 
aggregation of  BAFs is an open problem,  and a challenging one mainly because there are different semantic interpretations of 
\FT{support} in BAFs \FT{(e.g., as deductive or necessary \cite{CayrolL13,BoellaGTV10,NouiouaR10})}
, which might generate ``inconsistencies'' in the aggregated framework. 

Moving from the considerations above, we advance the state of the art on the application of Social Choice Theory to opinion aggregation in computational argumentation
by focusing on BA. To address the problem that different parties may adopt different interpretations of support in their opinions (BAFs), we use \textit{Bipolar Assumption-based Argumentation (ABA) frameworks} \cite{Cyras0T17} for representing opinions. Bipolar ABA is a restricted (but \FT{``}non-flat\FT{''}) 
form of ABA providing a unified formalism to accommodate different interpretations of support \cite{Cyras0T17}. Thus, by adopting Bipolar ABA, we let parties choose their interpretation of support before aggregation takes place.
Then, our contribution is twofold. 
Firstly, we define aggregation procedure for Bipolar ABA frameworks based on Social Choice Theory. Our investigations mainly focus on quota and oligarchic rules \cite{Grossi2014}. 
Secondly, we conduct a study on the preservation of properties using the defined aggregation procedures. In some cases, restrictions need to be placed, for example, towards specific aggregation rules.

{\bf Structure of the Paper.} 
In Section \ref{Section: Bipolar ABA} we provide the necessary background on Bipolar ABA. Section \ref{Section: Social Choice Theory} sets the ground, by formulating the aggregation problem, aggregation rules and preservation properties in our  Bipolar ABA setting. 
Section \ref{Section: Preservation Results} 
\FT{gives} the main theoretical contribution of the paper
\FT{, providing} preservation results for various properties of Bipolar ABA frameworks. Finally, Section \ref{Section: Conclusion} concludes and elaborates on several promising directions for future works.
%
%
Because of space limit, we omit some proofs: these  can be found in the Appendix.


\section{Bipolar 
ABA }{\label{Section: Bipolar ABA}}
Bipolar Assumption-based Argumentation \cite{Cyras0T17} (Bipolar ABA
) is a form of structured argumentation, where arguments and attacks are derived from assumptions, rules, and a contrary map from assumptions. \FT{Note that} contrary should not be confused with negation, which may or may not occur in the \FT{underlying} language \FT{\cite{ABAtut}}.

\begin{definition}
\cite{Cyras0T17} A {\em Bipolar ABA framework} is a 
	\FT{quadruple} $\langle \mathcal{L}, \mathcal{R}, \mathcal{A}, \contraries \rangle$, where
    \begin{itemize}
	    \item  ($\mathcal{L}$, $\mathcal{R}$) is a \emph{deductive system} with $\mathcal{L}$ a {\em language} (i.e. a set of sentences) and $\mathcal{R}$ a set of {\em rules} of the form $\phi \leftarrow \alpha$ where $\alpha \in \mathcal{A}$ and either $\phi \in \mathcal{A}$ or $\phi = \contrary{\beta}$ for some $\beta \in \mathcal{A}$; $\phi$ is the {\em head} and $\alpha$ the {\em body} of rule $\phi \leftarrow \alpha$;
    
    \item  $\mathcal{A} \subseteq \mathcal{L}$ is a non-empty set of {\em assumptions};

    \item  $\contraries: \mathcal{A} \rightarrow \mathcal{L}$ is a total map; for $\alpha \in \mathcal{A}$, $\contrary{\alpha}$ is the {\em contrary} of $\alpha$.
    \end{itemize}
\end{definition}

\FT{Then, a {\em deduction} for $\phi \in \mathcal{L}$ supported by $A \subseteq \mathcal{A}$ and $R \subseteq \mathcal{R}$, denoted $A \vdash^R \phi$, is a finite tree with the root labelled by $\phi$; leaves labelled by assumptions, with $A$ 
the set of all such assumptions; and each non-leaf node $\psi$ has a single child labelled by the body of some $\psi$-headed rule in $\mathcal{R}$, with $R$ 
the set of all such rules.}

Note that in Bipolar ABA rules are of a restricted kind, in comparison with generic ABA~\cite{ABA97}: their bodies amount to a single assumption, and thus, in particular, there are no rules with an empty body; also, their heads are either assumptions or contraries thereof.  Because assumptions may be ``deducible'' from rules in Bipolar ABA frameworks, though, these frameworks may be \emph{non-flat} in general, thus lacking some of the properties that \emph{flat} ABA frameworks (where assumptions are not ``deducible'' from rules) exhibit 
\cite{CyrasFST17}.


In Bipolar ABA, $A \subseteq \mathcal{A}$ {\em attacks} $\beta \in \mathcal{A}$ iff $\exists A' \vdash^R \contrary{\beta}$, such that $A' \subseteq A$\FT{;} $\alpha \in \mathcal{A}$ {\em attacks} $\beta \in \mathcal{A}$ iff $\{\alpha\}$ attacks $\beta$\FT{;} $A \subseteq \mathcal{A}$ {\em attacks} $B \subseteq \mathcal{A}$ iff $\exists \beta \in B$ such that $A$ attacks $\beta$. Then, $A$ is {\em conflict-free} iff $A$ does not attacks 
\FT{$A$}. 
\FT{Let} the {\em closure of $A$}\FT{$\subseteq \mathcal{A}$} be     $Cl(A) = \{\alpha \in \mathcal{A} : \exists A' \vdash^R \alpha, A' \subseteq A, R \subseteq \mathcal{R}\}$. Then, $A$ is {\em closed} iff $A = Cl(A)$. 

Several conditions can be imposed on Bipolar ABA frameworks, which characterise different sets   of assumptions (also called \emph{extensions}) according to as many {\em semantics}. Table \ref{Table:Semantics} gives the semantics for Bipolar ABA frameworks we will analyse in the paper, in addition to properties of conflict-freeness and closedness.
\begin{table}
\begin{center}
\caption{Bipolar ABA Semantics (for extension $A \subseteq \mathcal{A}$). \label{Table:Semantics}}
\begin{tabular}{lp{6cm}}
\toprule
\multicolumn{1}{c}{\textbf{Semantics}} & \multicolumn{1}{c}{\textbf{Conditions}} \\ \midrule

{Admissible} & $A$ is closed, conflict-free and for every $B \subseteq \mathcal{A}$, if $B$ is closed and attacks $A$, then $A$ attacks $B$.\\ 

{Preferred} & $A$ is $\subseteq$-maximally admissible. \\ 

{Complete} & $A$ is admissible and $A=\{\alpha \in \mathcal{A} : A$ defends $\alpha \}$ where $A$ {\em defends} $\alpha \in \mathcal{A}$ iff for all closed $B \subseteq \mathcal{A}$: if $B$ attacks $\alpha$ then $A$ attacks $B$. \\ 

	{Set-stable} & $A$ is closed, conflict-free, and attacks $Cl(\beta)$ for each $\beta \in \mathcal{A} \setminus A$.  \\ 

{Well-founded} & $A$ is the intersection of all complete extensions. \\ 

{Ideal} & $A$ is $\subseteq$-maximal such that it is admissible and 
	$A \subseteq B$ for all preferred extensions $B \subseteq \mathcal{A}$. \\ \bottomrule
\end{tabular}
\end{center}
\end{table}
\FT{As a simple illustration of these semantics, consider a  Bipolar ABA framework with $\mathcal{L}=\{\alpha, \beta, \gamma, \contrary{\alpha}, \contrary{\beta}, \contrary{\gamma}\}$,  $\mathcal{A}=\{\alpha, \beta, \gamma\}$ and $\mathcal{R}=\{\contrary{\beta} \leftarrow \gamma, \gamma \leftarrow \alpha\}$.\footnote{With an abuse of notation, we use $\contrary{x}$ to denote the $\mathcal{L}$-sentence amounting to the contrary of $x \in \mathcal{A}$ and omit to specify the contrary map explicitly.} Then, $\{\alpha\}$ and $\{\alpha,\beta \}$ are not closed (and thus not admissible etc.), $\{\beta\}$ is closed and conflict-free but not admissible etc., and  $\{\alpha,\gamma\}$ is closed, conflict-free, admissible (etc.). }
\FT{Note that,} in Bipolar ABA (as in flat ABA, but not in general ABA~\cite{CyrasFST17}), admissible, preferred and ideal extensions are guaranteed to exist: in particular, since rules cannot have an empty body, the empty set of assumption is closed, and thus admissible. Instead, complete, well-founded and set-stable extensions may not exist \FT{\cite{Cyras0T17}}. 


Bipolar ABA provides an all-encompassing framework for capturing different interpretations of support (under admissible, preferred and set-stable semantics) \cite{Cyras0T17}, as illustrated in the following example, adapted from \cite{Rago2017}.\footnote{The formal definitions of how different interpretations of support are captured in Bipolar ABA are in \cite{Cyras0T17}, and outside the scope of this paper.}

\begin{example}{\label{Example: Bipolar ABA}}
The UK public may hold a range of views on Brexit:

         A: The UK should leave the EU.

         B: The UK staying in the EU is good for its economy.

         C: The EU's immigration policies are bad for the UK's economy.

         D: EU membership fees are too high.

         E: The UK staying in the EU is good for world peace.
\\
	Here 
	$A$ may be deemed to be attacked by 
	$B$ and $E$, but supported by $C$ and $D$. A Bipolar ABA representation for the \emph{deductive interpretation} of support~\cite{CayrolL13}
	is $\langle \mathcal{L}, \mathcal{R}, \mathcal{A}, \contraries \rangle$, where

\begin{itemize}
    \item $\mathcal{L} = \{A, B, C, D, E, \contrary{A}, \contrary{B}, \contrary{C}, \contrary{D}, \contrary{E}\}$;
    \item $\mathcal{R} = \{\contrary{A} \leftarrow B, \quad \contrary{A} \leftarrow E, \quad A \leftarrow C, \quad A \leftarrow D\}$;
    \item $\mathcal{A} = \{A, B, C, D, E\}$.
\end{itemize}
	Instead, a Bipolar ABA representation for the \emph{necessary interpretation} of support~\cite{NouiouaR10} is $\langle \mathcal{L}, \mathcal{R}', \mathcal{A}, \contraries \rangle$ with $\mathcal{R}' = \{\contrary{A} \leftarrow B, \quad \contrary{A} \leftarrow E, \quad C \leftarrow A, \quad D \leftarrow A\}$. (Other interpretations of support can also be represented but are 
	not illustrated here for lack of space).
\end{example}

Note that Bipolar ABA can naturally capture \emph{supported attacks} under the deductive interpretation of support~\cite{CayrolL13}, for example the supported attack from $\alpha$ to $\beta$ in a (standard) BAF where $\alpha$ supports $\gamma$ and $\gamma$ attacks $\beta$ is matched by $\alpha$ attacking $\beta$ in a Bipolar ABA framework where $\{\contrary{\beta} \leftarrow \gamma, \gamma \leftarrow \alpha\}\subseteq \mathcal{R}$.


\section{Social Choice 
for Bipolar ABA}\label{Section: Social Choice Theory}
Social choice theory mainly focuses on how to aggregate (people's or agents') opinions into a single collective decision. Broadly speaking, there are mainly two types of aggregation: preference aggregation \cite{Christian2013, Rothe2016} and judgement aggregation \cite{Christian2013} (other aggregation types exist but are omitted here). Here we focus on the latter, and 
adapt notions given by \cite{Ulle2017} to accommodate opinions drawn from Bipolar ABA frameworks. 

Hereafter we assume a set of agents $N = \{1, \ldots, n\}$ ($n >1$), with agents' opinions represented as Bipolar ABA frameworks. 
We assume that the (Bipolar ABA frameworks of) agents have the same language, assumptions, and contraries. Thus, the aggregation combines the rules of the frameworks. 
This is intuitive because 
aggregating 
the Bipolar ABA frameworks means aggregating the attacks and supports in the original BAFs, which correspond to the rules in Bipolar ABA frameworks.
We also assume that agents behave ``rationally'', for example 
that their opinions are not ``self-attacking'' 
, 
\FT{e.g.,} rules $\contrary{\alpha} \leftarrow \alpha$ (for $\alpha \in \mathcal{A}$) will never belong to $\mathcal{R}$; further 
\FT{elaboration} on agents' rationality, 
when they are defined argumentatively, can be found in \cite{Rago2017}.

%
%
To collectively combine the opinions of all agents, i.e., the agents' Bipolar ABA frameworks, we need aggregation rules. 

\begin{definition}
	{\label{Definition: Bipolar ABA aggregation}}
	Let $\mathbfcal{F}$ be the set of all Bipolar ABA frameworks with the same language $\mathcal{L}$, set of assumptions $\mathcal{A}$ and contrary mapping $\contraries$. 
	A \emph{Bipolar ABA aggregation rule} is a mapping $F : \mathbfcal{F}^n \to \mathbfcal{F}$ from $n$ Bipolar ABA frameworks into a single Bipolar ABA framework.
	Given $n$ (as opinions of agents in $N = \{1, \ldots ,n\}$) Bipolar ABA frameworks $\big \langle \mathcal{L}, \mathcal{R}_1, \mathcal{A}, \contraries \big \rangle,$ $\ldots,$ $\big \langle \mathcal{L}, \mathcal{R}_n, \mathcal{A}, \contraries \big \rangle$, $F$ returns 
	a single \emph{aggregated Bipolar ABA framework} $\big \langle \mathcal{L}, \mathcal{R}_{agg}, \mathcal{A}, \contraries \big \rangle$.
\end{definition}

Inspired by graph aggregation \cite{Grandi2017}, we restrict attention in this paper to
aggregation rules in the form of  
either quota rules or oligarchic rules, defined below in our setting.

\paragraph{Quota Rules.}
These set a quota $q \in N$ as a threshold to accept some set of rules $R \subseteq \mathbfcal{R} = \bigcup_{i \in N} \mathcal{R}_i$, i.e., there should be at least $q$ agents that accept the rules in $R$.

\begin{definition}
	The \emph{quota rule} 
	$F_q$, for $q \in N$,  is a Bipolar ABA aggregation rule such that \(F_q(
	\big \langle \mathcal{L}, \mathcal{R}_1, \mathcal{A}, \contraries \big \rangle, \ldots,\big \langle \mathcal{L}, \mathcal{R}_n, \mathcal{A}, \contraries \big \rangle)\!= \!\{ r \in \mathbfcal{R} : r \in \bigcap_{j \in N'} \mathcal{R}_j \text{ for } N'\subseteq N,\; |N'| \geq q\}.
	\)

\end{definition}

The quota $q$ can be any number, but there are several special quotas that are commonly used: 
\emph{weak majority} has a quota $q = \lfloor \frac{n}{2} \rfloor$;
\emph{strict majority} has a quota $q = \lceil \frac{n}{2} \rceil$;
\emph{nomination} accepts all rules accepted by at least 1 agent, i.e., $q = 1$; and \emph{unanimity} requires all agents to accept the same rules,
i.e., $q = n$.

For example, assume that $n=3$ and agents accept sets of rules $\mathcal{R}_1 = \{\contrary{A} \leftarrow B\}$, $\mathcal{R}_2 = \{A \leftarrow C\}$, and $\mathcal{R}_3 = \{\contrary{A} \leftarrow B, \quad A \leftarrow D\}$. Using weak majority, the aggregated Bipolar ABA framework has set of rules $\mathcal{R}_{agg} = \{\contrary{A} \leftarrow B, \quad A \leftarrow C, \quad A \leftarrow D\}$; with strict majority, $\mathcal{R}_{agg} = \{\contrary{A} \leftarrow B\}$; nomination gives $\mathcal{R}_{agg} = \{\contrary{A} \leftarrow B, \quad A \leftarrow C, \quad A \leftarrow D\}$, while unanimity returns $\mathcal{R}_{agg} = \{\}$.

\paragraph{Oligarchic Rules.}
These give agents the power to veto the accepted rule sets.  Clearly, oligarchic rules are not fair in that the opinions of agents without veto power are disregarded. However, in some cases they are necessary to avoid conflicts among the agents.

\begin{definition}
	Let $N_v \subseteq N$ be the \emph{agents with veto power}. 
	The \emph{oligarchic rule} $F_o$ is a Bipolar ABA aggregation rule such that \(F_o(
	\big \langle \mathcal{L}, \mathcal{R}_1, \mathcal{A}, \contraries \big \rangle, \ldots,\big \langle \mathcal{L}, \mathcal{R}_n, \mathcal{A}, \contraries \big \rangle)\!= \!\{ r \in \mathbfcal{R} : r \in \bigcap_{j \in N_v} \mathcal{R}_j \}.
	\)
	If $|N_v|=1$ 
	then the oligarchic rule is called \emph{dictatorship}.
\end{definition}

If all agents have veto powers, then the  oligarchic rule coincides with unanimity. As an example, assume, as above, that $n = 3$ and agents accept sets of rules $\mathcal{R}_1 = \{\contrary{A} \leftarrow B\}$, $\mathcal{R}_2 = \{A \leftarrow C\}$, and $\mathcal{R}_3 = \{\contrary{A} \leftarrow B, \quad A \leftarrow D\}$. If agents 1 and 3 are given veto powers, then 
the aggregated Bipolar ABA framework has set of rules 
$\mathcal{R}_{agg} = \{\contrary{A} \leftarrow B\}$. On the other hand, if all agents have veto powers, then $\mathcal{R}_{agg} = \{\}$, as with unanimity.

We will study whether the properties of the agents' Bipolar ABA frameworks, including semantics, conflict-freeness and closedness, and other ``graph'' properties, are preserved in the aggregated Bipolar ABA framework. To produce stronger preservation results, we assume that a property under consideration for an aggregated Bipolar ABA framework (obtained by applying some Bipolar ABA aggregation rule) needs to be satisfied by all agents.

\begin{definition}
	{\label{Definition: Preservation}}
    Let $P$ be a property of Bipolar ABA frameworks. 
	If $\Delta\subseteq \mathcal{A}$ is $P$ in each agents' Bipolar ABA framework $\big \langle \mathcal{L}, \mathcal{R}_i, \mathcal{A}, \contraries \big \rangle$ (with $i \in N$), then $P$ is \emph{preserved} in 
	the aggregated Bipolar ABA framework $\mathcal{F}=\big \langle \mathcal{L}, \mathcal{R}_{agg}, \mathcal{A}, \contraries \big \rangle$
	if and only if $\Delta$ is $P$ in $\mathcal{F}$. 
\end{definition}

\balance


\section{Preservation Results}{\label{Section: Preservation Results}}
In this section, we present preservation results for several properties $P$ (see Definition~\ref{Definition: Preservation}), specifically for $P$ equal to conflict-freeness and closedness of sets of assumptions, $P$ any of the semantics in Table~\ref{Table:Semantics}, $P$ amounting to assumption acceptability under these semantics, and $P$ amounting to 
(implicative and disjunctive) ``graph'' properties adapted from \cite{Grandi2017}.
Throughout, we will assume that  $\mathcal{F}=\big \langle \mathcal{L}, \mathcal{R}_{agg}, \mathcal{A}, \contraries \big \rangle$ is the aggregated Bipolar ABA framework resulting from the Bipolar ABA aggregation rules as considered, and, for any such rule, we will say that the rule preserves a property $P$ to mean that $P$ is preserved in  $\mathcal{F}$, as specified in Definition~\ref{Definition: Preservation}.

\subsection{Conflict-freeness}

Conflict-freeness is a basic property in argumentation, always preserved in our setting. 

\begin{theorem}{\label{Theorem: conflict-freeness}}
    Every quota rule and oligarchic rule preserves con-flict-freeness.
\end{theorem}

\begin{proof}
	Assume that $\Delta\subseteq \mathcal{A}$ is conflict-free in $\big \langle \mathcal{L}, \mathcal{R}_i, \mathcal{A}, \contraries \big \rangle$ for all $i \in N$. By contradiction, assume that $\Delta$ is not conflict-free in $\mathcal{F}$. Then $\exists \alpha, \beta \in \Delta$ such that $\alpha$ attacks $\beta$, i.e., 
	$\exists R=\{\contrary{\beta} \leftarrow \gamma_1, \; \ldots, \; \gamma_{m-1} \leftarrow \gamma_m\}\subseteq \mathcal{R}_{agg}$ for $m \geq 1$, $\{\gamma_1, \ldots, \gamma_m\} \subseteq \mathcal{A}$ and $\gamma_m=\alpha$. By definition of quota and oligarchic rules, there has to be at least one agent $i \in N$ 
	such that $R \subseteq \mathcal{R}_i$ and thus $\alpha$ attacks $\beta$ in $\big \langle \mathcal{L}, \mathcal{R}_i, \mathcal{A}, \contraries \big \rangle$, 
	thus contradicting 
	our assumption that $\Delta$ is conflict-free in $\big \langle \mathcal{L}, \mathcal{R}_i, \mathcal{A}, \contraries \big \rangle$ for all $i \in N$.
\end{proof}

Note that Theorem \ref{Theorem: conflict-freeness} is a direct extension of Theorem 2 in \cite{Ulle2017} because conflict-freeness in Bipolar ABA frameworks holds under the same conditions as in AAFs. 

\subsection{Closedness}

Closedness is an important property in Bipolar ABA frameworks,
made so by the presence of ``support'' between assumptions (in the form of rules whose head and body are assumptions).
Instead, it is not meaningful in AA, and indeed it is not studied in \cite{Ulle2017}.

\begin{theorem}{\label{Theorem: closure preservation}}
	Every quota 
	\FT{and} oligarchic rule preserve\FT{s} closedness.
\end{theorem}

\begin{proof}
	Assume that $\Delta\subseteq \mathcal{A}$ is closed in $\big \langle \mathcal{L}, \mathcal{R}_i, \mathcal{A}, \contraries \big \rangle$ for all $i \in N$. By contradiction, assume that $\Delta$ is not closed in 
	$\mathcal{F}$. 
	Then $\exists \alpha \in \Delta$ and $\beta \not\in \Delta$ such that $\beta \in Cl(\{\alpha\})$, i.e. there exists  
	$R=\{\beta \leftarrow \gamma_1, \; \ldots, \; \gamma_{m-1} \leftarrow \gamma_m\}\subseteq \mathcal{R}_{agg}$ for $m \geq 1$, $\{\gamma_1, \ldots, \gamma_m\} \subseteq \mathcal{A}$ and $\gamma_m=\alpha$. By definition of quota and oligarchic rules, there has to be at least one agent $i \in N$  such that $R \subseteq \mathcal{R}_i$ and thus $\beta \in Cl(\{\alpha\})$ in $\big \langle \mathcal{L}, \mathcal{R}_i, \mathcal{A}, \contraries \big \rangle$, 
 thus contradicting 
	our assumption that $\Delta$ is closed in $\big \langle \mathcal{L}, \mathcal{R}_i, \mathcal{A}, \contraries \big \rangle$ for all $i \in N$.
	
\end{proof}

\subsection{Admissible Extensions}

The preservation result below (Theorem~\ref{Theorem: Admissibility}) for admissibility extends Theorem 3 in \cite{Ulle2017} by considering also support. It assumes constraints on the number of assumptions. Theorem \ref{Theorem: admissibility less than 3 assumptions} below instead analyses the preservation of admissibility for corner cases. 
\begin{theorem}{\label{Theorem: Admissibility}}
    For $|\mathcal{A}| \geq 4$, nomination  is the only quota rule that preserves admissibility.
\end{theorem}
\begin{proof}
	First we prove that nomination preserves admissibility.
	Assume that $\Delta\subseteq \mathcal{A}$ is admissible in $\big \langle \mathcal{L}, \mathcal{R}_i, \mathcal{A}, \contraries \big \rangle$ for all agents $i \in N$. By contradiction, assume that $\Delta$ is not admissible in 
	$\mathcal{F}$. 
	Then $\exists \alpha \in \Delta$ that is attacked by $\beta \in \mathcal{A} \setminus \Delta$, i.e., $R=\{\contrary{\alpha} \leftarrow \gamma_1, \; \ldots, \; \gamma_{m-1} \leftarrow \gamma_m \}\subseteq \mathcal{R}_{agg}$, for $m \geq 1$ and $\gamma_m=\beta$, and 
	$\not\exists \gamma \in \Delta$ such that $\gamma$ attacks $\beta$
	in $\mathcal{F}$.
	By definition of nomination rule, 
	$R \subseteq \mathcal{R}_i$ for some $i\in N$ and $\beta$ attacks $\alpha$ in the Bipolar ABA framework $\mathcal{F}_i$ of agent $i$. Then, given that $\Delta$ is admissible in $\mathcal{F}_i$,  $\exists \gamma \in \Delta$ such that $\gamma$ attacks $\beta$ in $\mathcal{F}_i$, i.e.,
	 $\exists R'=\{\contrary{\beta} \leftarrow \delta_1, \; \ldots, \; \delta_{l-1} \leftarrow \delta_l \}\subseteq \mathcal{R}_{i}$, for $l \geq 1$ and $\delta_l=\gamma$. 	
	 But, by definition of nomination rule, $R' \subseteq \mathcal{R}_{agg}$, and thus $\gamma$ attacks $\beta$ in $\mathcal{F}$: contradiction.
    To complete the proof, we need to show that for $|\mathcal{A}| \geq 4$, other quota rules except for nomination do not preserve admissibility. If $N - q$ agents choose rule $\mathcal{R} = \{\}$, $q-1$ agents choose rule $\mathcal{R} = \{\contrary{D} \leftarrow B, \quad \contrary{C} \leftarrow D\}$, and 1 agent chooses rule $\mathcal{R} = \{\contrary{D} \leftarrow A, \quad \contrary{C} \leftarrow D, \quad A \leftarrow B\}$, as illustrated in Figure \ref{fig:Admissibility}, then  
	$\Delta = \{A,B,C\}$ is admissible in all frameworks. Using quota rules with $q > 1$, 
	\FT{$\mathcal{R}_{agg} = \{\contrary{C} \leftarrow D\}$} and $\Delta$ is not admissible anymore as assumption $C$ is attacked by $D$ and it is undefended by other assumptions in $\Delta$. 
\end{proof}

\begin{figure}
    \centerline{\includegraphics[width = 0.42\textwidth]{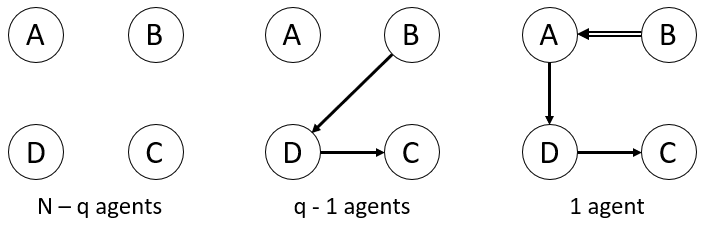}}
	\caption{Counter Example for the Preservation of Admissibility (for Theorem~\ref{Theorem: Admissibility}). Here, we use BAFs as a graphical representation of Bipolar ABA frameworks (assuming the deductive interpretation of support). Single-lined edges are attacks, double-lined edges are supports.}
    \label{fig:Admissibility}
    \Description{Three graphical representations of Bipolar ABA frameworks to show the counter example of the preservation of admissible extension.}
\end{figure}

We remark that for at most three assumptions, 
quota and oligarchic rules are guaranteed to preserve admissibility.

\begin{theorem}{\label{Theorem: admissibility less than 3 assumptions}}
    For $|\mathcal{A}| \leq 3$, every quota rule and oligarchic rule preserves admissibility.
\end{theorem}

\begin{proof} 
	If $|\mathcal{A}| = 1$, the result holds vacuously. 
	If $|\mathcal{A}| = 2$, assume that $\mathcal{A} = \{\alpha, \beta\}$ 
	and $\Delta = \{\alpha\}$ 
	is admissible in $\big \langle \mathcal{L}, \mathcal{R}_i, \mathcal{A}, \contraries \big \rangle$ for all $i \in N$. Then $r=\contrary{\alpha} \leftarrow \beta \notin \mathcal{R}_i$ for all $i \in N$, and thus $r \not\in \mathcal{R}_{agg}$. Then, every quota and oligarchic rule yields $\{\alpha\}$ as an admissible extension in $\mathcal{F}$.
	The other cases ($\Delta = \{\beta\}$, $\Delta = \{\}$ or $\Delta=\mathcal{A}$) can be proven similarly.
    If $|\mathcal{A}| = 3$, assume that $\mathcal{A} = \{\alpha, \beta, \gamma\}$. Consider the case where $\Delta = \{\alpha\}$ 
	is admissible in $\big \langle \mathcal{L}, \mathcal{R}_i, \mathcal{A}, \contraries \big \rangle$ for all $i \in N$. By contradiction, assume that $\Delta$ is not admissible in $\mathcal{F}$. Then, given that $\Delta$ is conflict-free and closed in $\mathcal{F}$ no matter which aggregation rule, by Theorems~\ref{Theorem: conflict-freeness} and \ref{Theorem: closure preservation}, there are 
	$R \subseteq \mathcal{R}_{agg}$ and $A \subseteq \{\beta, \gamma\}$ such that $A \vdash^R \contrary{\alpha}$ in $\mathcal{F}$. By quota and oligarchic rules, there must be $i \in N$ such 
	that $R \subseteq \mathcal{R}_{i}$; thus $\Delta$ is not admissible 
	for agent $i$: contradiction.
	The other cases can be proven similarly.
\end{proof}

Note that the restriction to consider at most three assumptions may be useful in some settings, e.g., when at most three options are up for debate.

\subsection{Set-stable Extensions}

The set-stable semantics for Bipolar ABA frameworks generalises the stable semantic for AAFs to accommodate supports (see \cite{Cyras0T17}). This generalisation though does not affect preservation. 
Thus, Theorem \ref{Theorem: set-stable} extends Proposition 5 in \cite{Ulle2017}.

\begin{theorem}{\label{Theorem: set-stable}}
    Nomination is the only quota rule that preserves set-stable extensions.
\end{theorem}

\begin{proof}
	Assume that $\Delta\subseteq \mathcal{A}$ is set-stable in $\big \langle \mathcal{L}, \mathcal{R}_i, \mathcal{A}, \contraries \big \rangle$ for all $i \in N$. By Theorem~\ref{Theorem: conflict-freeness} and \ref{Theorem: closure preservation}, nomination preserves closedness and conflict-freeness. Therefore, $\Delta$ is both closed and conflict-free in 
	$\mathcal{F}$. To be set-stable, $\Delta$ has to attack 
	the closure of every assumption $\beta$ not $\Delta$, i.e., 
	$\exists R  \subseteq \mathcal{R}_{agg}$ such that $\{\alpha\} \vdash^R \contrary{\gamma}$ for some $\gamma \in Cl(\{\beta\})$.
This is trivially the case if $\Delta = \mathcal{A}$. 
	Otherwise, as $\Delta$ is set-stable in all agents' frameworks, then 
	$R_i\subseteq \mathcal{R}_i$ must exist, for all $i \in N$, such that $\{\alpha\} \vdash^{R_i} \contrary{\gamma}$ for some $\gamma \in Cl(\{\beta\})$ in $\big \langle \mathcal{L}, \mathcal{R}_{i}, \mathcal{A}, \contraries \big \rangle$. 
	Thus, by using nomination, $Cl(\{\beta\})$, for $\beta \in \mathcal{A} \setminus \Delta$, is attacked also in $\mathcal{F}$.
    
	Other quota rules do not preserve set-stable extensions because, 
	\FT{while preserving} conflict-freeness and closedness, they do not guarantee that the closure of every assumption not in the extension is attacked. A counter example 
	\FT{follows}: assume three Bipolar ABA frameworks with rules $\mathcal{R}_1 = \{\contrary{D} \leftarrow B, \quad B \leftarrow A\}$, $\mathcal{R}_2 = \{\contrary{D} \leftarrow C\}$, and $\mathcal{R}_3 = \{\contrary{D} \leftarrow A, \quad \contrary{C} \leftarrow D, \quad A \leftarrow B\}$, as illustrated in Figure \ref{fig:Set-stable}. In each framework, the set of assumptions $\{A,B,C\}$ is set-stable. Using other quota rules with $q > 1$,
	$\mathcal{R}_{agg} = \{\}$, and $\{A,B,C\}$ is not set-stable anymore as the assumption $D$ is not included in it and it is not attacked either ($Cl(\{D\})=\{D\}$ in this example).
\end{proof}

\begin{figure}
    \centerline{\includegraphics[width = 0.42\textwidth]{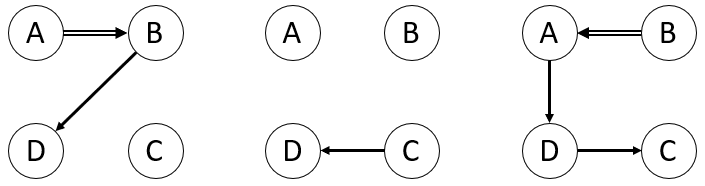}}
	\caption{Counter Example of the Preservation of Set-stable Extensions (for Theorem~\ref{Theorem: set-stable}).
	Here, we use BAFs as a graphical representation of Bipolar ABA frameworks (assuming the deductive interpretation of support). Single-lined edges are attacks, double-lined edges are supports.}
    \label{fig:Set-stable}
    \Description{Three graphical representations of Bipolar ABA frameworks to show the counter example of the preservation of set-stable extension.}
\end{figure}

\subsection{Assumption Acceptability}{\label{Section: The acceptability of an assumption}}

\FT{Assumption acceptability} concerns the preferred, complete, set-stable, well-founded, and ideal semantics 
\FT{but at the level of single assumption} rather than 
\FT{full extensions}. If an assumption is acceptable in one of those semantics (by belonging to a set of assumptions accepted by the semantics) in each of the agents' framework, then preservation amounts to that assumption being still acceptable under the same semantics in the aggregated Bipolar ABA framework.

\begin{definition}(Acceptability of Assumptions)
	An assumption $\alpha\in \mathcal{A}$ is 
	\emph{acceptable} (in a Bipolar ABA framework) under preferred, complete, set-stable, well-founded, or ideal semantic iff 
	there 
	\FT{is} $\Delta\!\subseteq\!\mathcal{A}$ with $\alpha \!\in \!\Delta$ such that $\Delta$ is (respectively) a preferred, complete, set-stable, well-founded, or ideal extension
	(in the Bipolar ABA framework).
\end{definition}

The proof of preservation results regarding acceptability 
use adaptations of 
results from \cite{Grandi2017} on \emph{implicative} and \emph{disjunctive} properties to represent impossibility results with dictatorship. 
We cast these properties for Bipolar ABA frameworks as follows:

\begin{definition}{(Implicative Properties).}{\label{Definition: Implicative}}
    A Bipolar ABA framework property $P$ is \emph{implicative} in $\big \langle \mathcal{L}, \mathcal{R}, \mathcal{A}, \contraries \big \rangle$ iff there exist three rules $R_1, R_2, R_3 \notin \mathcal{R}$ such that $P$ holds in 
	$\big \langle \mathcal{L}, \mathcal{R}_{agg}, \mathcal{A}, \contraries \big \rangle$ for $\mathcal{R}_{agg} = \mathcal{R} \cup \mathcal{S}$, for all $\mathcal{S} \subseteq \{R_1, R_2, R_3\}$, except for $\mathcal{S} = \{R_1, R_2\}$.
\end{definition}

Intuitively, 
if the aggregated Bipolar ABA framework includes $R_1$ and $R_2$ as additional rules in $\mathcal{S}$, then it should adopt $R_3$ as well to preserve property $P$.

\begin{definition}{(Disjunctive Properties).}{\label{Definition: Disjunctive}}
    A Bipolar ABA framework property $P$ is \emph{disjunctive} in $\big \langle \mathcal{L}, \mathcal{R}, \mathcal{A}, \contraries \big \rangle$ iff there exist two rules $R_1, R_2 \notin \mathcal{R}$, such that $P$ holds in 
	$\big \langle \mathcal{L}, \mathcal{R}_{agg}, \mathcal{A}, \contraries \big \rangle$ for $\mathcal{R}_{agg} = \mathcal{R} \cup \mathcal{S}$, for all $\mathcal{S} \subseteq \{R_1, R_2\}$, except for $\mathcal{S} = \{\}$.
\end{definition}

Intuitively, 
the aggregated Bipolar ABA framework 
has to include at least one of $R_1$ or $R_2$ to preserve the property $P$. 
Definitions of implicative and disjunctive properties lead us to prove two lemmas on preservation.

\begin{lemma}{\label{Lemma: Oligarchy}}
	Let a Bipolar ABA framework property $P$ be implicative in $\big \langle \mathcal{L}, \mathcal{R}_{i}, \mathcal{A}, \contraries \big \rangle$, for each $i \in N$. Then,
    unanimity preserves $P$.
\end{lemma}

\begin{proof}
Let 
	$\mathcal{R}_1 \!\supseteq \!\mathcal{S}_1, \ldots, \mathcal{R}_n \!\supseteq \!\mathcal{S}_n$ and let
	$\mathcal{S}_i \!\subseteq \!\{R_1,R_2,R_3\}$ for all $i\! \in \!N$.
	Let $\mathcal{S}_i \!\neq \!\{R_1, R_2\}$ for all $i \!\in \!N$. Then, unanimity preserves $P$ because it is impossible to get $\mathcal{S}_{agg} \!= \!\{R_1, R_2\}$, with $\mathcal{R}_{agg} \!\supseteq \!\mathcal{S}_{agg}$.
\end{proof}

Note that, even if $P$ is implicative, nomination and majority do not preserve $P$ in general, as it is possible to get $\mathcal{S}_{agg} = \{R_1, R_2\}$. For example, let $\mathcal{S}_1 = \{R_1\}$, and $\mathcal{S}_2 = \{R_2\}$. Using nomination and majority,  $\mathcal{R}_{agg} \supseteq$ $\mathcal{S}_{agg} = \{R_1, R_2\}$; hence, $P$ is not guaranteed to be preserved.

\begin{lemma}{\label{Lemma: Dictatorial}}
	Let a Bipolar ABA framework property $P$ be implicative and disjunctive 
	in $\big \langle \mathcal{L}, \mathcal{R}_{i}, \mathcal{A}, \contraries \big \rangle$, for each $i \in N$. Then, the only Bipolar ABA aggregation rule that preserves $P$ 
	is dictatorship.
\end{lemma}

\begin{proof}
    The proof for the implicativeness can be found in Lemma \ref{Lemma: Oligarchy}. For the disjunctiveness, let $\mathcal{R}_1 \supseteq \mathcal{S}_1, \ldots,$ $ \mathcal{R}_n \supseteq \mathcal{S}_n$ and let
	$\mathcal{S}_i \subseteq \{R_1,R_2\}$ for all $i \in N$. Let $\mathcal{S}_i \neq \{\}$ for all $i \in N$. Then, nomination and majority preserve $P$ because it is impossible to get $\mathcal{S}_{agg} = \{\}$, with $\mathcal{R}_{agg} \supseteq \mathcal{S}_{agg}$.
    As $P$ is implicative and disjunctive, $P$ is preserved only with dictatorship. None of the quota rules preserve $P$ as using nomination or majority rule, it is possible to get $\mathcal{S}_{agg} = \{R_1,R_2\}$ and violating the implicativeness; and using unanimity rule, it is possible to get $\mathcal{S}_{agg} = \{\}$ and violating the disjunctiveness.
\end{proof}

Notice that, in the definition of  implicative and disjunctive properties, the rules $R_1, R_2$, and (if applicable) $R_3$ can only be in the form of $\contrary{\alpha} \leftarrow \beta$ for some $\alpha, \beta \in \mathcal{A}$, thus bringing attacks between assumptions. They cannot be in the form of $\alpha \leftarrow \beta$ that denote supports between assumptions because then some agents may have different closures of assumptions from the other agents. As a consequence, some agents' Bipolar ABA frameworks may satisfy $P$, while some others may not because of closedness. 

The preservation result on the acceptability of an assumption in Theorem~\ref{Theorem: Acceptability of An Assumption} below is an extension, within our more general setting, of Theorem 1 in \cite{Ulle2017}. This result is true for all five semantics: preferred, complete, set-stable, well-founded, or ideal. If the agents' Bipolar ABA frameworks have no (rules for) support, then this Theorem \ref{Theorem: Acceptability of An Assumption} is the same as Theorem 1 in \cite{Ulle2017}.

\begin{theorem}{\label{Theorem: Acceptability of An Assumption}}
	For $|\mathcal{A}| \geq 4$, the only Bipolar ABA aggregation rule that preserves the acceptability of an assumption under preferred, complete, set-stable, well-founded, or ideal semantic
	is dictatorship.
\end{theorem}

\begin{proof}
    Let 
	$P$ be acceptability of an assumption under preferred, complete, set-stable, well-founded, or ideal semantics. We need to prove that for $|\mathcal{A}| \geq 4$, $P$ is implicative and disjunctive. Then, by Lemma~\ref{Lemma: Dictatorial}, the theorem holds. 
	The proof has the same structure for each of the five semantics. Consider 
	a set of at least four assumptions $\mathcal{A} = \{A,B,C,D,\ldots\}$.
    
	To show that $P$ is implicative, let $B$ be the accepted assumption. Let $\mathcal{R} = \{\contrary{C} \leftarrow A, \quad D \leftarrow A\}$, $R_1 = \{\contrary{B} \leftarrow C\}$, $R_2 = \{\contrary{A} \leftarrow B\}$, and $R_3 = \{\contrary{C} \leftarrow D\}$ (see the left graph of Figure \ref{fig:Acceptability of an Assumption}). Consider an aggregated framework  
	with $\mathcal{R}_{agg} = \mathcal{R} \cup \mathcal{S}$ with $\mathcal{S} \subseteq \{R_1, R_2, R_3\}$. 
	If $\mathcal{S} = \{\}$, $\{R_2\}$, $\{R_3\}$, or $\{R_2, R_3\}$ then 
	$B$ is unattacked. If $\mathcal{S} = \{R_1\}$, $\{R_1, R_3\}$, or $\{R_1, R_2, R_3\}$ then   
	$B$ is defended by other assumptions. Therefore, 
	$B$ is either unattacked or defended in all seven cases, and $B$ is acceptable under preferred, complete, set-stable, well-founded, and ideal semantics. However, if $\mathcal{S} = \{R_1, R_2\}$, 
	$\{A,B,C\}$ forms cyclic attacks so that the assumptions $A,B$, and $C$ are not acceptable under preferred, complete, set-stable, well-founded, and ideal semantics. Thus, we have identified 
	a set of rules $\mathcal{R}$ and three rules $R_1, R_2, R_3$ such that $P$ holds in $\big \langle \mathcal{L}, \mathcal{R} \cup \mathcal{S}, \mathcal{A}, \contraries \big \rangle$ iff $\mathcal{S} \neq \{R_1, R_2\}$. Accordingly, $P$ is implicative.
    
	To show that $P$ is disjunctive, let $B$ be the accepted assumption. Let $\mathcal{R} = \{\contrary{B} \leftarrow A, \quad D \leftarrow C\}$, $R_1 = \{\contrary{A} \leftarrow C\}$, and $R_2 = \{\contrary{A} \leftarrow D\}$ ( 
	see the right graph of Figure \ref{fig:Acceptability of an Assumption}). Consider 
	$\mathcal{R}_{agg} = \mathcal{R} \cup \mathcal{S}$ with $\mathcal{S} \subseteq \{R_1, R_2\}$. If $\mathcal{S} = \{R_1\}$, $\{R_2\}$, or $\{R_1, R_2\}$ then the assumption $B$ is defended. Therefore, $B$ is acceptable under the five semantics. However, if $\mathcal{S} = \{\}$, the assumption $B$ is attacked by $A$ and is not defended, thus $B$ is unacceptable under preferred, complete, set-stable, well-founded, and ideal semantics. Thus, 
	we have identified a set of rules $\mathcal{R}$ and two rules $R_1, R_2$ such that $P$ holds in $\big \langle \mathcal{L}, \mathcal{R} \cup \mathcal{S}, \mathcal{A}, \contraries \big \rangle$ iff $\mathcal{S} \neq \{\}$. Therefore, $P$ is  
	disjunctive. 
\end{proof}

\begin{figure}
    \centerline{\includegraphics[width = 0.27\textwidth]{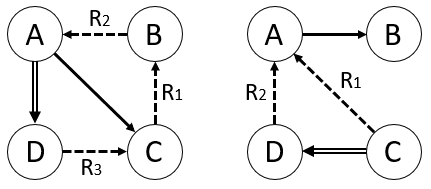}}
	\caption{Acceptability of an Assumption: Implicative case (Left) and Disjunctive case (Right) (for the proof of Theorem~\ref{Theorem: Acceptability of An Assumption}). Here, we use BAFs as a graphical representation of Bipolar ABA frameworks (assuming the deductive interpretation of support). Single-lined \FT{and hyphenated edges} are attacks, double-lined edges are supports. 
	}
    \label{fig:Acceptability of an Assumption}
    \Description{Two graphical representations of Bipolar ABA frameworks to show the acceptability of an assumption is both implicative and disjunctive.}
\end{figure}

Theorem \ref{Theorem: Acceptability of An Assumption} shows that it is not easy to even preserve the acceptability of one assumption, as dictatorship is needed. We will see, in Section \ref{Section: Preferred,Complete,Well-founded,Ideal} below, that it is even more difficult to preserve whole extensions. On the other hand, for $|\mathcal{A}| \leq 3$,  
acceptability of an assumption can be preserved with both quota and oligarchic rules.

\begin{theorem}{\label{Theorem: Acceptability of an Assumption equal 3}}
    For $|\mathcal{A}| = 3$, majority, unanimity, and oligarchic rules preserve assumption acceptability under preferred, complete, set-stable, well-founded, and ideal semantics.
\end{theorem}

\begin{proof}
    Let a Bipolar ABA framework property $P$ be the acceptability of an assumption under preferred, complete, set-stable, well-founded, or ideal semantics. Assume that $P$ holds in $\big \langle \mathcal{L}, \mathcal{R}_i, \mathcal{A}, \contraries \big \rangle$ for all $i \in N$, where $\mathcal{A} = \{\alpha, \beta, \gamma\}$ and assume that $\alpha$ is acceptable under preferred, complete, set-stable, well-founded, and ideal semantics in all frameworks.

	By contradiction, assume $P$ does not hold in $\mathcal{F}$. In other words, $\exists R  \subseteq \mathcal{R}_{agg}$ such that $\{\delta\} \vdash^R \contrary{\alpha}$ and not $\exists R  \subseteq \mathcal{R}_{agg}$ such that $\{\theta\} \vdash^R \contrary{\delta}$ for some $\theta \in \{\beta, \gamma\}$ and $\delta \in \{\beta, \gamma\}$, $\theta \neq \delta$. As a result, $\alpha$ is not acceptable under preferred, complete, set-stable, well-founded, and ideal semantics. 
	By definition of 
	of majority rule, unanimity rule, and oligarchic rules, the deduction from rules $\{\delta\} \vdash^R \contrary{\alpha}$ must exist in the majority (majority rule), all (unanimity rule), or veto powered (oligarchic rules) agents' frameworks, but there is at least one framework $\big \langle \mathcal{L}, \mathcal{R}_i, \mathcal{A}, \contraries \big \rangle$ for some $i \in N$ without the deduction $\{\theta\} \vdash^R \contrary{\delta}$ for $\theta \neq \delta$, thus $\alpha$ is not acceptable under preferred, complete, set-stable, well-founded, and ideal semantics in the agents' frameworks as well. It contradicts the initial assumption that $\alpha$ is acceptable in $\big \langle \mathcal{L}, \mathcal{R}_i, \mathcal{A}, \contraries \big \rangle$ for all $i \in N$.
    
    To show that for $|\mathcal{A}| = 3$, nomination rule do not preserve the assumption acceptability under preferred, complete, set-stable, well-founded, or ideal semantics, a counter example is given. Take three Bipolar ABA frameworks with rules $\mathcal{R}_1 = \{\contrary{A} \leftarrow C, \quad \contrary{B} \leftarrow C\}$, $\mathcal{R}_2 = \{\contrary{B} \leftarrow A, \quad \contrary{C} \leftarrow B\}$, and $\mathcal{R}_3 = \{\contrary{C} \leftarrow A, \quad \contrary{A} \leftarrow B\}$. Let assumption $C$ be the accepted assumption in check. From the first framework, a set of assumptions $\{C\}$ is preferred, complete, set-stable, well-founded, and ideal. On the second framework is $\{A,C\}$ and third framework is $\{B,C\}$, both extensions are preferred, complete, set-stable, well-founded, and ideal as well. In all three frameworks, the assumption $C$ is acceptable. It is still acceptable using unanimity rule as the aggregated rule is $\mathcal{R} = \{\}$. However, using nomination rule the preferred, complete, and set-stable extensions are $\{A\}, \{B\}$, and $\{C\}$; while the well-founded and ideal extensions are $\{\}$. Hence, 
	\FT{$\{C\}$} is not acceptable in those five semantics.
\end{proof}

\begin{theorem}{\label{Theorem: Acceptability of an Assumption less than 2}}
    For $|\mathcal{A}| \leq 2$, every quota and oligarchic rule preserves the acceptability of an assumption under preferred, complete, set-stable, well-founded, and ideal semantics.
\end{theorem}

\subsection{Preferred, Complete, Well-founded, and Ideal Extensions}{\label{Section: Preferred,Complete,Well-founded,Ideal}}

The proof of preservation for the preferred, complete, well-founded, and ideal semantics uses the concept of implicative and disjunctive properties from Section \ref{Section: The acceptability of an assumption}. The preservation result is an extension of Theorem 4 in \cite{Ulle2017} with the addition of ideal semantics.

\begin{theorem}{\label{Theorem: Preferred and Complete}}
	For $|\mathcal{A}| \geq 5$, the only Bipolar ABA aggregation rule that preserves preferred, complete, well-founded, and ideal semantics  
	is dictatorship.
\end{theorem}

\begin{proof}
	Let $P$ be that an extension $\Delta\subseteq \mathcal{A}$ is preferred, complete, well-founded, or ideal. We need to prove that for $|\mathcal{A}| \geq 5$, $P$ is implicative and disjunctive. 
	Then by Lemma \ref{Lemma: Dictatorial}, the theorem holds.
	The proof has the same structure for all four semantics. It uses
	a generic $\mathcal{A} = \{A,B,C,D,E,\ldots\}$ with at least five assumptions.
    
    To show that $P$ is implicative, let $\Delta = \{B,D,E\}$. Let $\mathcal{R} = \{\contrary{C} \leftarrow D, \quad  \contrary{A} \leftarrow B, \quad E \leftarrow D\}$, $R_1 = \{\contrary{B} \leftarrow C\}$, $R_2 = \{\contrary{D} \leftarrow A\}$, and $R_3 = \{\contrary{A} \leftarrow E\}$ 
	(see the left graph of Figure \ref{fig:Preferred and Complete}). Consider  
	$\mathcal{R}_{agg} = \mathcal{R} \cup \mathcal{S}$ with $\mathcal{S} \subseteq \{R_1, R_2, R_3\}$. If $\mathcal{S} = \{\}, \{R_1\}, \{R_2\}, \{R_3\}, \{R_1, R_3\}, \{R_2, R_3\}$, or $\{R_1, R_2, R_3\}$.  Then $\Delta$ is preferred, complete, well-founded, and ideal. However, if $\mathcal{S} = \{R_1, R_2\}$, 
	$\{A, B, C, D\}$ forms cyclic attacks such that $\Delta$ is not preferred, complete, well-founded, or ideal. Thus,  
	we have identified a set of rules $\mathcal{R}$ and three rules $R_1, R_2, R_3$ such that $P$ holds in $\big \langle \mathcal{L}, \mathcal{R} \cup \mathcal{S}, \mathcal{A}, \contraries \big \rangle$ iff $\mathcal{S} \neq \{R_1, R_2\}$. Hence, $P$ is implicative.
    
    To show that $P$ is disjunctive, let $\Delta = \{B,D,E\}$. 
	Let $\mathcal{R} = \{\contrary{C} \leftarrow D, \quad \contrary{B} \leftarrow C, \quad \contrary{A} \leftarrow B, \quad \contrary{D} \leftarrow A, \quad D \leftarrow E\}$, $R_1 = \{\contrary{C} \leftarrow E\}$, and $R_2 = \{\contrary{A} \leftarrow E\}$ (see the right graph of Figure \ref{fig:Preferred and Complete}). Consider 
	$\mathcal{R}_{agg} = \mathcal{R} \cup \mathcal{S}$ with $\mathcal{S} \subseteq \{R_1, R_2\}$. If $\mathcal{S} = \{R_1\}, \{R_2\},$ or $\{R_1, R_2\}$, then $\Delta$ is preferred, complete, well-founded, and ideal. However, if $\mathcal{S} = \{\}$, 
	$\{A, B, C, D\}$ forms cyclic attacks such that $\Delta$ is not preferred, complete, well-founded, or ideal. Therefore, 
	we have identified
	a set of rules $\mathcal{R}$ and two rules $R_1, R_2$ such that $P$ holds in $\big \langle \mathcal{L}, \mathcal{R} \cup \mathcal{S}, \mathcal{A}, \contraries \big \rangle$ iff $\mathcal{S} \neq \{\}$. Thus, $P$ is disjunctive.
\end{proof}

\begin{figure}
    \centerline{\includegraphics[width = 0.37\textwidth]{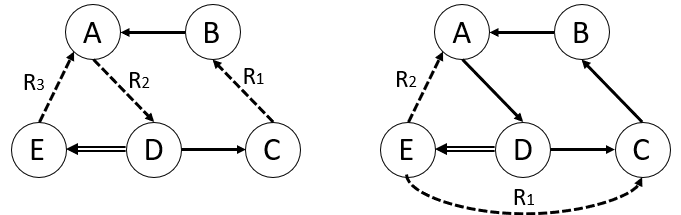}}
	\caption{Preferred, Complete, Well-founded, and Ideal extensions: Implicative case (Left) and Disjunctive case (Right) (for the proof of Theorem~\ref{Theorem: Preferred and Complete}). Here, we use BAFs as a graphical representation of Bipolar ABA frameworks (assuming the deductive interpretation of support). Single-lined \FT{and hyphenated} edges are attacks, double-lined edges are supports. 
	}
    \label{fig:Preferred and Complete}
    \Description{Two graphical representations of Bipolar ABA frameworks to show the preferred, complete, well-founded, and ideal extensions are both implicative and disjunctive.}
\end{figure}

Although preferred and complete semantics may accept multiple extensions, as long as all agents agree on the extensions, then Theorem \ref{Theorem: Preferred and Complete} still holds. The only restriction is the presence of supports in the agents' frameworks. If all agents agree on the supports, i.e., the supports are included in $\mathcal{R}_i$ for all $i \in N$, then 
support does not affect the preservation. Otherwise, some agents will have different sets of closed assumptions from the other agents. This may lead into different 
preferred, complete, well-founded, and ideal extensions.

The corner cases show an impossibility result for $|\mathcal{A}| = 3$ and $|\mathcal{A}| = 4$. Thus, preserving whole extensions is more difficult than preserving the acceptability of single assumptions as in Section \ref{Section: The acceptability of an assumption}.

\begin{theorem}{\label{Theorem: Preferred and Complete equal 4}}
    For $|\mathcal{A}| = 3$ and $|\mathcal{A}| = 4$, quota and oligarchic rules do not preserve preferred, complete, well-founded, and ideal semantics.
\end{theorem}

\begin{theorem}{\label{Theorem: Preferred and Complete less than 2}}
    For $|\mathcal{A}| \leq 2$, every quota and oligarchic rule preserves preferred, complete, well-founded, and ideal semantics.
\end{theorem}

\subsection{Non-emptiness of Well-founded Extensions}

The well-founded extension is guaranteed to exist in a Bipolar ABA framework. However, to make sure that the well-founded extension is not empty, then the framework must have at least one unattacked assumption. This way, the unattacked assumptions are included in all complete extensions, and the intersection always has the unattacked assumptions in it. The preservation of non-emptiness of the well-founded extension 
guarantees 
the existence of unattacked assumption with a concept called \textit{k-exclusivity} \cite{Ulle2017}.

\begin{definition}{(k-exclusivity).}
    Let $P$ be a property of Bipolar ABA framework. $P$ is \textit{k-exclusive} if there exist rules $\mathcal{S} = \{R_1, \ldots, R_k\}$ such that if $\mathcal{R} \supseteq \mathcal{S}$ then $P$ does not hold, but if $\mathcal{R} \subset \mathcal{S}$ then $P$ holds.
\end{definition}

Thus, to preserve $P$, the rules $\mathcal{S}$ cannot be adopted together, but only a subset of them. It leads us to the lemma for the preservation.

\begin{lemma}
    Let $P$ be a \textit{k-exclusive} property of Bipolar ABA framework. For $k \geq N$, where $N$ is the number of agents, $P$ is preserved if at least one of the $N$ agents has veto power.
\end{lemma}

\begin{proof}
    It needs to be showed that if an aggregation rule preserves $P$, then it has to give at least one agent with veto powers. Notice that if all agents accept a rule $r$, then it must be accepted in the aggregated rules, i.e., $r \in \mathcal{R}_{agg}$ iff $r \in \mathcal{R}_i$ for all $i \in N$.
    
    For some agents $i \in N$ to have veto powers means that $\mathcal{R}_{agg} = (\bigcap \mathcal{R}_i)$. In other words, some agents have veto power, if the intersection of the agents' rules in $\bigcap \mathcal{R}_i$ are all accepted in $\mathcal{R}_{agg}$. Then, take any rule $r \in \mathcal{R}_{agg}$; as $r$ is accepted in the aggregated framework, then all agents with veto powers must accept $r$ as well such that the intersection of the set of rules $\bigcap \mathcal{R}_i$ is not empty.
    
	Thus, the next step is to show that if an aggregation rule preserves $P$, then the intersection of $k$ set of rules must be non-empty, i.e., $\mathcal{R}_1 \cap \ldots \cap \mathcal{R}_k \neq \{\}$. To prove by contradiction, assume there exist a profile of set of rules $\{\mathcal{R}_1 \cup \ldots \cup \mathcal{R}_k\} \subseteq \mathcal{R}_{agg}$ such that $\mathcal{R}_1 \cap \ldots \cap \mathcal{R}_k = \{\}$. Then, it means that for every $j \in \{1, \ldots, k\}$, exactly (the agent with rule set) $\mathcal{R}_j$ accepts a rule $r_j$. As no rule exist in all $\mathcal{R}_i$ for $i \in N$, no agents accept all $k$ rules. However, as each of the $k$ rules is accepted by an agent and $\{\mathcal{R}_1 \cup \ldots \cup \mathcal{R}_k\} \subseteq \mathcal{R}_{agg}$, they are all accepted in the aggregated framework, i.e., $\{r_1, \ldots, r_k\} \subseteq \mathcal{R}_{agg}$, such that $P$ does not hold due to it being an \textit{k-exclusive} property. This contradicts the initial assumption that the aggregation rule preserves $P$.
    
    Therefore, as it can be showed that the intersection of the agents' rules is not empty, then some agents must have veto powers.
\end{proof}

\begin{theorem}{\label{Theorem: The non-emptiness of well-founded}}
    For $|\mathcal{A}| \geq N$, at least one agent must have veto power to preserve the non-emptiness of the well-founded extension.
\end{theorem}

\begin{proof}
    Let a Bipolar ABA framework property $P$ be the non-emptiness of
    the well-founded extension. We need to show that $P$
    is \textit{k-exclusive}. Let $k = |\mathcal{A}|$ and
	$\{A_1, \ldots, A_k\} \subseteq \mathcal{A}$. Assume that $\mathcal{S}$ consists of all rules $\contrary{A_{i+1}} \leftarrow A_i$ for $i < |\mathcal{A}|$ as well as $\contrary{A_1} \leftarrow A_k$, illustrated in Figure \ref{fig:Non-emptiness of well-founded extension}. This $\mathcal{S}$  fits the definition of \textit{k-exclusive}. Indeed, if $\mathcal{S} \subseteq \mathcal{R}$, then in the case of $\mathcal{S} = \mathcal{R}$, the well-founded extension is empty due to the cyclic attacks. However, if only a subset of it is adopted\FT{,} $\mathcal{R} \subset \mathcal{S}$, the well-founded extension is not empty as at least one assumption is not attacked. Thus, $P$ is preserved when at least one agent has veto power to prevent cyclic attacks.
\end{proof}

\begin{figure}
    \centering
    \includegraphics[width = 0.17\textwidth]{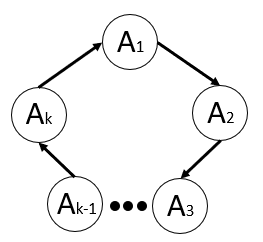}
    \caption{Graphical illustration of the \textit{k-exclusivity} property (for the proof of Theorem~\ref{Theorem: The non-emptiness of well-founded})}
    \label{fig:Non-emptiness of well-founded extension}
    \Description{A graphical representation of Bipolar ABA frameworks to show the k-exclusivity property for the non-emptiness of well-founded extensions}
\end{figure}

\FT{S}upports in Bipolar ABA framework do not affect the preservation of 
non-emptiness of \FT{the} well-founded extension because supports between assumptions do not 
\FT{affect} the unattacked assumption\FT{:} if there 
\FT{is} a rule $\contrary{\alpha} \leftarrow \beta$ for $\alpha, \beta \in \mathcal{A}$ and $\beta$ is unattacked, then supports from and into $\beta$ do not change the fact that $\beta$ is unattacked; and supports from and into $\alpha$ also leave $\beta$ unattacked.

\subsection{Acyclicity}

It is clear that \textit{k-exclusivity} deals with cyclic attacks. A Bipolar ABA framework is acyclic if there does not exist any cyclic attacks among the assumptions. Corollary \ref{Corollary: Acyclicity} is extended from Theorem 8 in \cite{Ulle2017} in the way that supports are also considered.

\begin{definition}(Cyclic \FT{Attacks})
    The rule set $\mathcal{R}$ in
    $\big \langle \mathcal{L}, \mathcal{R}, \mathcal{A}, \contraries \big \rangle$
	\emph{contain cyclic attacks} if there exist a chained connection between
	some of the assumptions in $\mathcal{A}$, such that $\mathcal{R}
    \supseteq \{\contrary{\alpha_1} \leftarrow \alpha_2 , \contrary{\alpha_2} \leftarrow \alpha_3, \ldots, \contrary{\alpha_k} \leftarrow \alpha_1\}$ for
    $\alpha_i \in \mathcal{A}$ and $k \geq 2$.
\end{definition}

The preservation result for acyclicity has a similar proof structure as the preservation of the non-emptiness of the well-founded extension in Theorem \ref{Theorem: The non-emptiness of well-founded}. Thus, it is presented as a corollary.

\begin{corollary}{\label{Corollary: Acyclicity}}
    For $|\mathcal{A}| \geq N$
	, at least one agent must have veto power to preserve acyclicity.
\end{corollary}

\begin{proof}
    Let 
	$P$ be acyclicity. We need to show that $P$ is \textit{k-exclusive}. To get a cycle, a minimum number of two assumptions are required. Thus, let $k = |\mathcal{A}|\geq 2$ and $\{A_1, \ldots, A_k\} \subseteq \mathcal{A}$. Assume that the rule set  $\mathcal{S}$ consists of $ \contrary{A_{i+1}} \leftarrow A_i$ for $i < |\mathcal{A}|$ as well as $\contrary{A_1} \leftarrow A_k$, illustrated in Figure \ref{fig:Non-emptiness of well-founded extension}. This $\mathcal{S}$ fits the definition of \textit{k-exclusivity}. Indeed, if $\mathcal{S} \subseteq \mathcal{R}$, then in the case of $\mathcal{S} = \mathcal{R}$, the cyclic attacks remain in the framework. However, if only a subset of $\mathcal{S}$ is adopted ($\mathcal{R} \subset \mathcal{S}$), then the cyclic attacks are broken because at least one rule that connects the cycle disappears. Therefore, $P$ is preserved when at least one agent has veto power.
\end{proof}

The presence of supports does not make an acyclic framework to be\FT{come} cyclic, but instead may break any existing cycle. Let $k = |\mathcal{A}|$ with $|\mathcal{A}| \geq 2$, $\{A_1, \ldots, A_k\} \subseteq \mathcal{A}$, and $\mathcal{S} = \{\contrary{A_{i+1}} \leftarrow A_i$: $i < |\mathcal{A}|\}$. The rules in $\mathcal{S}$ are acyclic and if a support $A_1 \leftarrow A_k$ or $A_k \leftarrow A_1$ is added, for example, then they will remain acyclic. On the contrary, if there exist cyclic attacks, then 
support may break the cycle due to 
closedness
.

\subsection{Coherence}

Coherence amounts to two or more semantics coinciding (in other words, given a Bipolar ABA framework, two or more semantics give identical extensions thereof). For example, if a set of assumptions is set-stable, then it is preferred as well. Our next preservation result extends Theorem~9 in \cite{Ulle2017} and shows that, in order to preserve coherence, the aggregation rule must be dictatorial. The proof for the result uses the concept of implicativeness and disjunctiveness introduced in Section \ref{Section: The acceptability of an assumption}.

\begin{theorem}{\label{Theorem: Coherence}}
    For $|\mathcal{A}| \geq 4$, the only aggregation rule preserving coherence is dictatorship.
\end{theorem}

\begin{proof}
    Let 
	$P$ be coherence. We need
    to prove that, for $|\mathcal{A}| \geq 4$, $P$ is implicative and
    disjunctive. Take a Bipolar ABA framework with at least four
    assumptions $\mathcal{A} = \{A,B,C,D,\ldots\}$.
    
    To show that $P$ is implicative, let $\mathcal{R}
    = \{\contrary{C} \leftarrow A, \quad D \leftarrow A\}$, $R_1
    = \{\contrary{B} \leftarrow C\}$, $R_2 = \{\contrary{A} \leftarrow
    B\}$, and $R_3 = \{\contrary{C} \leftarrow D\}$, as illustrated in
    the left graph of Figure \ref{fig:Coherence}. Consider an
    aggregated framework with
    $\mathcal{S} \subseteq \{R_1, R_2, R_3\}$. If $\mathcal{S}
    = \{\}$, $\{R_1\}$, $\{R_3\}$, or $\{R_1, R_3\}$, the only
    preferred extension is $\{A,B,D\}$, which is set-stable as well. If
    $\mathcal{S} = \{R_2\}$, the set of assumptions $\{B,C,D\}$ is both
    preferred and set-stable. If $\mathcal{S} = \{R_2, R_3\}$ or $\{R_1,
    R_2, R_3\}$; then the set of assumptions $\{B,D\}$ is both
    preferred and set-stable as well. However, if $\mathcal{S} = \{R_1,
    R_2\}$, the only preferred extension is $\{D\}$ and it is not
    set-stable as the other assumptions are not attacked. Thus, there
    exists a set of rules $\mathcal{R}$ and three rules $R_1, R_2,
    R_3$ such that $P$ holds in
    $\big \langle \mathcal{L}, \mathcal{R} \cup \mathcal{S}, \mathcal{A}, \contraries \big \rangle$
    iff $\mathcal{S} \neq \{R_1, R_2\}$. Accordingly, $P$
    is an implicative property.
    
    To show that $P$ is disjunctive, let $\mathcal{R} = \{\contrary{A} \leftarrow D, \quad \contrary{B} \leftarrow A, \quad \contrary{D} \leftarrow B, \quad C \leftarrow A\}$, $R_1 = \{\contrary{D} \leftarrow C\}$, and $R_2 = \{\contrary{B} \leftarrow C\}$, as illustrated in the right graph of Figure \ref{fig:Coherence}. Consider an aggregated framework $\big \langle \mathcal{L}, \mathcal{R}_{agg}, \mathcal{A}, \contraries \big \rangle$, where $\mathcal{R}_{agg} = \mathcal{R} \cup \mathcal{S}$ with $\mathcal{S} \subseteq \{R_1, R_2\}$. If $\mathcal{S} = \{R_1\}$ or $\{R_1, R_2\}$, the set of assumptions $\{A,C\}$ is both preferred and set-stable. If $\mathcal{S} = \{R_2\}$, the set of assumptions $\{C,D\}$ is also preferred and set-stable. However, if $\mathcal{S} = \{\}$, the preferred extension is $\{C\}$ and it is not set-stable because the other assumptions are not attacked. Therefore, there exists a set of rules $\mathcal{R}$ and two rules $R_1, R_2$ such that $P$ holds in $\big \langle \mathcal{L}, \mathcal{R} \cup \mathcal{S}, \mathcal{A}, \contraries \big \rangle$ iff $\mathcal{S} \neq \{\}$. Hence, $P$ is a disjunctive property.
    
    As $P$ is proven to be both implicative and disjunctive, then by Lemma \ref{Lemma: Dictatorial}, for $P$ to be preserved, the aggregation rule must be dictatorial.
\end{proof}

\begin{figure}
    \centering
    \includegraphics[width = 0.27\textwidth]{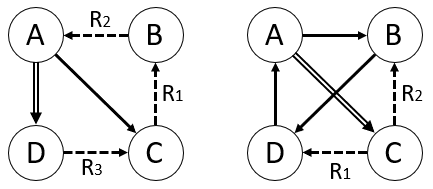}
	\caption{Coherence: Implicative case (Left) and Disjunctive case (Right) (for the proof of Theorem~\ref{Theorem: Coherence}). Here, we use BAFs as a graphical representation of Bipolar ABA frameworks (assuming the deductive interpretation of support). Single-lined \FT{and hyphenated} edges are attacks, double-lined edges are supports. 
	}
    \label{fig:Coherence}
    \Description{Two graphical representations of Bipolar ABA frameworks to show that the coherence property is both implicative and disjunctive.}
\end{figure}

Note that Theorem \ref{Theorem: Coherence} also works for other semantics (indeed, in the proof, the accepted sets of assumptions may be complete and well-founded, rather than just preferred and set-stable).

The presence of supports is acceptable in the preservation of coherence only if the supports are adopted by each agent, such that all agents have the same closure of assumptions. If supports join the additional rules in $\mathcal{S}$ as either $R_1, R_2,$ or $R_3$, then coherence is not preserved in the aggregated framework as some agents have different set of closures than the other
agents.

For corner cases, it is easier to preserve coherence\FT{, as indeed} 
unanimity 
preserves it for $|\mathcal{A}| \leq 3$. Moreover, admittedly less interestingly, both quota and oligarchic rules preserve coherence when there is one assumption.

\begin{theorem}{\label{Theorem: Coherence 2 and 3}}
    For $|\mathcal{A}| = 2$ or $|\mathcal{A}| = 3$, unanimity rule is the only quota rule that preserves coherence.
\end{theorem}

\begin{theorem}{\label{Theorem: Coherence 1}}
    For $|\mathcal{A}| = 1$, every quota rule and oligarchic rule preserves coherence.
\end{theorem}

\section{Conclusion}{\label{Section: Conclusion}}
We have considered Bipolar ABA Frameworks \cite{Cyras0T17} to account for both attack and support relationships between arguments in Bipolar Argumentation, as it allows to capture uniformly different 
\FT{interpretations} of support. 
The aggregation of Bipolar ABA Frameworks combines the rules of all agents into a collective set of rules. We made use of the aggregation rules from judgement aggregation \cite{Christian2013, Grossi2014}, specifically, quota and oligarchic rules, to combine these agents' rules 
and extended results (on Abstract Argumentation) from \cite{Ulle2017}. Generally, the preservation results show that most properties can be preserved, but with significant restrictions sometimes. The results  
assume agreement among the agents 
on language, assumptions, contraries, and assume that agents accept the same properties (Definition\FT{s} \ref{Definition: Bipolar ABA aggregation} and \ref{Definition: Preservation}).
We observe that, when the notion of agreement comes into play, 
the presence of supports 
does not greatly affect the performance of the aggregation rules towards preservation.

As regards positive results, conflict-freeness and closedness are preserved by any quota and oligarchic rule (Theorem\FT{s} \ref{Theorem: conflict-freeness} and \ref{Theorem: closure preservation}). We proved positive results for admissibility and set-stability as well, albeit with some restrictions, such as limiting the number of assumptions or the choice of aggregation rules. Admissibility is preserved by nomination for at least four assumptions, else it is preserved by every quota and oligarchic rule; while the set-stable semantics is preserved by nomination (Theorem\FT{s} \ref{Theorem: Admissibility}, \ref{Theorem: admissibility less than 3 assumptions}, and \ref{Theorem: set-stable}).

We show that some properties can only be preserved by oligarchic rules or dictatorship. These particular aggregation rules are actually not ideal, as they ignore most opinions. However, we still deem this better than not being able to preserve the properties at all. For the properties of acceptability of an assumption,
and coherence
when the number of assumptions is at least four, dictatorship is the only preserving rule (Theorem\FT{s} \ref{Theorem: Acceptability of An Assumption}
 and \ref{Theorem: Coherence}). 
The same holds for preferred, complete, well-founded, and ideal semantics, but by assuming at least five assumptions (Theorem \ref{Theorem: Preferred and Complete}). Unsurprisingly, in the corner cases, these properties can be preserved with other quota rules (Theorem\FT{s} \ref{Theorem: Acceptability of an Assumption equal 3}, \ref{Theorem: Acceptability of an Assumption less than 2}, \ref{Theorem: Preferred and Complete equal 4}, \ref{Theorem: Preferred and Complete less than 2}
, \ref{Theorem: Coherence 2 and 3}, and \ref{Theorem: Coherence 1}).

Preservation results also involve the non-emptiness of the well-founded extension and acyclicity. We proved that both properties are preserved when at least one agent has veto power and the number of assumptions is greater or equal than the number of agents (Theorem \ref{Theorem: The non-emptiness of well-founded} and Corollary \ref{Corollary: Acyclicity}). This unique constraint is meant to avoid cyclic relationships. 

To conclude, our preservation study produces stronger results to fill the gaps in \cite{Ulle2017} since we consider more properties, some relevant to Bipolar Argumentation only (closedness) others also relevant to Abstract Argumentation  (ideal semantics); we also provide preservation results for corner cases. 

There are several possible directions for future work. First of all, here the preservation of properties relies on the agreement of all agents. However, in real applications it is likely that some agents have different opinions, i.e., some of them might disagree on the properties. Thus, it would be interesting to study preservation when a number of agents disagree. Another path to work on in the future is to have agents with different knowledge about the environment, meaning that they might have different languages, assumptions, or contraries.
A further promising direction for future work is to expand the choice of aggregation rules with a more complex formalisation. 
Finally, it would be worth to generalise this study for the more general ABA Frameworks of \cite{ABA97,CyrasFST17}\FT{, as well as for other forms of structured argumentation,  such as ASPIC~\cite{Modgil:Prakken:2014}, DeLP~\cite{Garcia:Simari:2014} or logic-based argumentation~\cite{Besnard:Hunter:2014}. In particular, the} 
possibility of having \FT{rules with} empty body 
might need specific attention 
\FT{when it comes to aggregation}.



\bibliographystyle{ACM-Reference-Format} 
\bibliography{bibliography}


\begin{thebibliography}{23}


\ifx \showCODEN    \undefined \def \showCODEN     #1{\unskip}     \fi
\ifx \showDOI      \undefined \def \showDOI       #1{#1}\fi
\ifx \showISBNx    \undefined \def \showISBNx     #1{\unskip}     \fi
\ifx \showISBNxiii \undefined \def \showISBNxiii  #1{\unskip}     \fi
\ifx \showISSN     \undefined \def \showISSN      #1{\unskip}     \fi
\ifx \showLCCN     \undefined \def \showLCCN      #1{\unskip}     \fi
\ifx \shownote     \undefined \def \shownote      #1{#1}          \fi
\ifx \showarticletitle \undefined \def \showarticletitle #1{#1}   \fi
\ifx \showURL      \undefined \def \showURL       {\relax}        \fi
\providecommand\bibfield[2]{#2}
\providecommand\bibinfo[2]{#2}
\providecommand\natexlab[1]{#1}
\providecommand\showeprint[2][]{arXiv:#2}

\bibitem[\protect\citeauthoryear{Baroni, Romano, Toni, Aurisicchio, and
  Bertanza}{Baroni et~al\mbox{.}}{2015}]%
        {BaroniRTAB15}
\bibfield{author}{\bibinfo{person}{Pietro Baroni}, \bibinfo{person}{Marco
  Romano}, \bibinfo{person}{Francesca Toni}, \bibinfo{person}{Marco
  Aurisicchio}, {and} \bibinfo{person}{Giorgio Bertanza}.}
  \bibinfo{year}{2015}\natexlab{}.
\newblock \showarticletitle{Automatic evaluation of design alternatives with
  quantitative argumentation}.
\newblock \bibinfo{journal}{\emph{Argument {\&} Computation}}
  \bibinfo{volume}{6}, \bibinfo{number}{1} (\bibinfo{year}{2015}),
  \bibinfo{pages}{24--49}.
\newblock


\bibitem[\protect\citeauthoryear{Baumeister and Rothe}{Baumeister and
  Rothe}{2016}]%
        {Rothe2016}
\bibfield{author}{\bibinfo{person}{Dorothea Baumeister} {and}
  \bibinfo{person}{J{\"{o}}rg Rothe}.} \bibinfo{year}{2016}\natexlab{}.
\newblock \showarticletitle{Preference Aggregation by Voting}.
\newblock In \bibinfo{booktitle}{\emph{Economics and Computation}}.
  \bibinfo{publisher}{Springer}, \bibinfo{address}{Berlin, Heidelberg},
  \bibinfo{pages}{197--325}.
\newblock


\bibitem[\protect\citeauthoryear{Besnard and Hunter}{Besnard and
  Hunter}{2014}]%
        {Besnard:Hunter:2014}
\bibfield{author}{\bibinfo{person}{Philippe Besnard} {and}
  \bibinfo{person}{Anthony Hunter}.} \bibinfo{year}{2014}\natexlab{}.
\newblock \showarticletitle{{Constructing Argument Graphs with Deductive
  Arguments: A Tutorial}}.
\newblock \bibinfo{journal}{\emph{Argument {\&} Computation}}
  \bibinfo{volume}{5}, \bibinfo{number}{1} (\bibinfo{year}{2014}),
  \bibinfo{pages}{5--30}.
\newblock
\showISSN{1946-2166}
\urldef\tempurl%
\url{https://doi.org/10.1080/19462166.2013.869765}
\showDOI{\tempurl}


\bibitem[\protect\citeauthoryear{Bodanza, Tohm{\'{e}}, and Auday}{Bodanza
  et~al\mbox{.}}{2017}]%
        {Tohme2017}
\bibfield{author}{\bibinfo{person}{Gustavo~Adrian Bodanza},
  \bibinfo{person}{Fernando Tohm{\'{e}}}, {and} \bibinfo{person}{Marcelo
  Auday}.} \bibinfo{year}{2017}\natexlab{}.
\newblock \showarticletitle{Collective argumentation: {A} survey of aggregation
  issues around argumentation frameworks}.
\newblock \bibinfo{journal}{\emph{Argument {\&} Computation}}
  \bibinfo{volume}{8}, \bibinfo{number}{1} (\bibinfo{year}{2017}),
  \bibinfo{pages}{1--34}.
\newblock


\bibitem[\protect\citeauthoryear{Boella, Gabbay, van~der Torre, and
  Villata}{Boella et~al\mbox{.}}{2010}]%
        {BoellaGTV10}
\bibfield{author}{\bibinfo{person}{Guido Boella}, \bibinfo{person}{Dov~M.
  Gabbay}, \bibinfo{person}{Leendert W.~N. van~der Torre}, {and}
  \bibinfo{person}{Serena Villata}.} \bibinfo{year}{2010}\natexlab{}.
\newblock \showarticletitle{Support in Abstract Argumentation}. In
  \bibinfo{booktitle}{\emph{Computational Models of Argument: Proceedings of
  {COMMA} 2010, Desenzano del Garda, Italy, September 8-10, 2010.}}
  \bibinfo{publisher}{{IOS} Press}, \bibinfo{address}{Desenzano del Garda,
  Italy}, \bibinfo{pages}{111--122}.
\newblock


\bibitem[\protect\citeauthoryear{Bondarenko, Dung, Kowalski, and
  Toni}{Bondarenko et~al\mbox{.}}{1997}]%
        {ABA97}
\bibfield{author}{\bibinfo{person}{Andrei Bondarenko},
  \bibinfo{person}{Phan~Minh Dung}, \bibinfo{person}{Robert~A. Kowalski}, {and}
  \bibinfo{person}{Francesca Toni}.} \bibinfo{year}{1997}\natexlab{}.
\newblock \showarticletitle{An Abstract, Argumentation-Theoretic Approach to
  Default Reasoning}.
\newblock \bibinfo{journal}{\emph{Artif. Intell.}}  \bibinfo{volume}{93}
  (\bibinfo{year}{1997}), \bibinfo{pages}{63--101}.
\newblock


\bibitem[\protect\citeauthoryear{Cayrol and Lagasquie{-}Schiex}{Cayrol and
  Lagasquie{-}Schiex}{2005}]%
        {CayrolL05a}
\bibfield{author}{\bibinfo{person}{Claudette Cayrol} {and}
  \bibinfo{person}{Marie{-}Christine Lagasquie{-}Schiex}.}
  \bibinfo{year}{2005}\natexlab{}.
\newblock \showarticletitle{On the Acceptability of Arguments in Bipolar
  Argumentation Frameworks}. In \bibinfo{booktitle}{\emph{Symbolic and
  Quantitative Approaches to Reasoning with Uncertainty, 8th European
  Conference, {ECSQARU} 2005, Barcelona, Spain, July 6-8, 2005, Proceedings}}.
  \bibinfo{publisher}{Springer}, \bibinfo{address}{Berlin, Heidelberg},
  \bibinfo{pages}{378--389}.
\newblock


\bibitem[\protect\citeauthoryear{Cayrol and Lagasquie{-}Schiex}{Cayrol and
  Lagasquie{-}Schiex}{2013}]%
        {CayrolL13}
\bibfield{author}{\bibinfo{person}{Claudette Cayrol} {and}
  \bibinfo{person}{Marie{-}Christine Lagasquie{-}Schiex}.}
  \bibinfo{year}{2013}\natexlab{}.
\newblock \showarticletitle{Bipolarity in argumentation graphs: Towards a
  better understanding}.
\newblock \bibinfo{journal}{\emph{Int. J. Approx. Reasoning}}
  \bibinfo{volume}{54}, \bibinfo{number}{7} (\bibinfo{year}{2013}),
  \bibinfo{pages}{876--899}.
\newblock


\bibitem[\protect\citeauthoryear{Chen and Endriss}{Chen and Endriss}{2017}]%
        {Ulle2017}
\bibfield{author}{\bibinfo{person}{Weiwei Chen} {and} \bibinfo{person}{Ulle
  Endriss}.} \bibinfo{year}{2017}\natexlab{}.
\newblock \showarticletitle{Preservation of Semantic Properties during the
  Aggregation of Abstract Argumentation Frameworks}. In
  \bibinfo{booktitle}{\emph{Proceedings Sixteenth Conference on Theoretical
  Aspects of Rationality and Knowledge, {TARK} 2017, Liverpool, UK, 24-26 July
  2017.}} \bibinfo{publisher}{OPA}, \bibinfo{address}{Liverpool, UK},
  \bibinfo{pages}{118--133}.
\newblock


\bibitem[\protect\citeauthoryear{Chen and Endriss}{Chen and Endriss}{2018}]%
        {Ulle2018}
\bibfield{author}{\bibinfo{person}{Weiwei Chen} {and} \bibinfo{person}{Ulle
  Endriss}.} \bibinfo{year}{2018}\natexlab{}.
\newblock \showarticletitle{Aggregating Alternative Extensions of Abstract
  Argumentation Frameworks: Preservation Results for Quota Rules}. In
  \bibinfo{booktitle}{\emph{Computational Models of Argument - Proceedings of
  {COMMA} 2018, Warsaw, Poland, 12-14 September 2018}}.
  \bibinfo{publisher}{{IOS} Press}, \bibinfo{address}{Warsaw, Poland},
  \bibinfo{pages}{425--436}.
\newblock


\bibitem[\protect\citeauthoryear{Cyras, Fan, Schulz, and Toni}{Cyras
  et~al\mbox{.}}{2017a}]%
        {CyrasFST17}
\bibfield{author}{\bibinfo{person}{Kristijonas Cyras}, \bibinfo{person}{Xiuyi
  Fan}, \bibinfo{person}{Claudia Schulz}, {and} \bibinfo{person}{Francesca
  Toni}.} \bibinfo{year}{2017}\natexlab{a}.
\newblock \showarticletitle{Assumption-based Argumentation: Disputes,
  Explanations, Preferences}.
\newblock \bibinfo{journal}{\emph{{FLAP}}} \bibinfo{volume}{4},
  \bibinfo{number}{8} (\bibinfo{year}{2017}), \bibinfo{pages}{2407--2455}.
\newblock


\bibitem[\protect\citeauthoryear{Cyras, Schulz, and Toni}{Cyras
  et~al\mbox{.}}{2017b}]%
        {Cyras0T17}
\bibfield{author}{\bibinfo{person}{Kristijonas Cyras}, \bibinfo{person}{Claudia
  Schulz}, {and} \bibinfo{person}{Francesca Toni}.}
  \bibinfo{year}{2017}\natexlab{b}.
\newblock \showarticletitle{Capturing Bipolar Argumentation in Non-flat
  Assumption-Based Argumentation}. In \bibinfo{booktitle}{\emph{{PRIMA} 2017:
  Principles and Practice of Multi-Agent Systems - 20th International
  Conference, Nice, France, October 30 - November 3, 2017, Proceedings}}.
  \bibinfo{publisher}{Springer}, \bibinfo{address}{Cham},
  \bibinfo{pages}{386--402}.
\newblock


\bibitem[\protect\citeauthoryear{Dung}{Dung}{1995}]%
        {Dung1995}
\bibfield{author}{\bibinfo{person}{Phan~Minh Dung}.}
  \bibinfo{year}{1995}\natexlab{}.
\newblock \showarticletitle{On the Acceptability of Arguments and its
  Fundamental Role in Nonmonotonic Reasoning, Logic Programming and n-Person
  Games}.
\newblock \bibinfo{journal}{\emph{Artif. Intell.}} \bibinfo{volume}{77},
  \bibinfo{number}{2} (\bibinfo{year}{1995}), \bibinfo{pages}{321--358}.
\newblock


\bibitem[\protect\citeauthoryear{Endriss and Grandi}{Endriss and
  Grandi}{2017}]%
        {Grandi2017}
\bibfield{author}{\bibinfo{person}{Ulle Endriss} {and} \bibinfo{person}{Umberto
  Grandi}.} \bibinfo{year}{2017}\natexlab{}.
\newblock \showarticletitle{Graph aggregation}.
\newblock \bibinfo{journal}{\emph{Artif. Intell.}}  \bibinfo{volume}{245}
  (\bibinfo{year}{2017}), \bibinfo{pages}{86--114}.
\newblock


\bibitem[\protect\citeauthoryear{Ganzer-Ripoll, Criado, Lopez-Sanchez, Parsons,
  and Rodriguez-Aguilar}{Ganzer-Ripoll et~al\mbox{.}}{2019}]%
        {Ganzer2019}
\bibfield{author}{\bibinfo{person}{Jordi Ganzer-Ripoll},
  \bibinfo{person}{Natalia Criado}, \bibinfo{person}{Maite Lopez-Sanchez},
  \bibinfo{person}{Simon Parsons}, {and} \bibinfo{person}{Juan~A.
  Rodriguez-Aguilar}.} \bibinfo{year}{2019}\natexlab{}.
\newblock \showarticletitle{{Combining Social Choice Theory and Argumentation:
  Enabling Collective Decision Making}}.
\newblock \bibinfo{journal}{\emph{Group Decision and Negotiation}}
  \bibinfo{volume}{28}, \bibinfo{number}{1} (\bibinfo{year}{2019}),
  \bibinfo{pages}{127--173}.
\newblock


\bibitem[\protect\citeauthoryear{Garc{\'{i}}a and Simari}{Garc{\'{i}}a and
  Simari}{2014}]%
        {Garcia:Simari:2014}
\bibfield{author}{\bibinfo{person}{Alejandro~Javier Garc{\'{i}}a} {and}
  \bibinfo{person}{Guillermo~Ricardo Simari}.} \bibinfo{year}{2014}\natexlab{}.
\newblock \showarticletitle{{Defeasible Logic Programming: DeLP-servers,
  Contextual Queries, and Explanations for Answers}}.
\newblock \bibinfo{journal}{\emph{Argument {\&} Computation}}
  \bibinfo{volume}{5}, \bibinfo{number}{1} (\bibinfo{year}{2014}),
  \bibinfo{pages}{63--88}.
\newblock
\urldef\tempurl%
\url{https://doi.org/10.1080/19462166.2013.869767}
\showDOI{\tempurl}


\bibitem[\protect\citeauthoryear{Grossi and Pigozzi}{Grossi and
  Pigozzi}{2014}]%
        {Grossi2014}
\bibfield{author}{\bibinfo{person}{Davide Grossi} {and}
  \bibinfo{person}{Gabriella Pigozzi}.} \bibinfo{year}{2014}\natexlab{}.
\newblock \bibinfo{booktitle}{\emph{Judgment Aggregation: {A} Primer}}.
\newblock \bibinfo{publisher}{Morgan {\&} Claypool Publishers},
  \bibinfo{address}{San Rafael}.
\newblock


\bibitem[\protect\citeauthoryear{List}{List}{2013}]%
        {Christian2013}
\bibfield{author}{\bibinfo{person}{Christian List}.}
  \bibinfo{year}{2013}\natexlab{}.
\newblock \showarticletitle{Social Choice Theory}.
\newblock In \bibinfo{booktitle}{\emph{The Stanford Encyclopedia of Philosophy}
  (\bibinfo{edition}{winter 2013} ed.)},
  \bibfield{editor}{\bibinfo{person}{Edward~N. Zalta}} (Ed.).
  \bibinfo{publisher}{Metaphysics Research Lab, Stanford University},
  \bibinfo{address}{Stanford}.
\newblock


\bibitem[\protect\citeauthoryear{Modgil and Prakken}{Modgil and
  Prakken}{2014}]%
        {Modgil:Prakken:2014}
\bibfield{author}{\bibinfo{person}{Sanjay Modgil} {and} \bibinfo{person}{Henry
  Prakken}.} \bibinfo{year}{2014}\natexlab{}.
\newblock \showarticletitle{{The ASPIC+ Framework for Structured Argumentation:
  A Tutorial}}.
\newblock \bibinfo{journal}{\emph{Argument {\&} Computation}}
  \bibinfo{volume}{5}, \bibinfo{number}{1} (\bibinfo{year}{2014}),
  \bibinfo{pages}{31--62}.
\newblock
\showISSN{1946-2166}
\urldef\tempurl%
\url{https://doi.org/10.1080/19462166.2013.869766}
\showDOI{\tempurl}


\bibitem[\protect\citeauthoryear{Nouioua and Risch}{Nouioua and Risch}{2010}]%
        {NouiouaR10}
\bibfield{author}{\bibinfo{person}{Farid Nouioua} {and}
  \bibinfo{person}{Vincent Risch}.} \bibinfo{year}{2010}\natexlab{}.
\newblock \showarticletitle{Bipolar Argumentation Frameworks with Specialized
  Supports}. In \bibinfo{booktitle}{\emph{22nd {IEEE} International Conference
  on Tools with Artificial Intelligence, {ICTAI} 2010, Arras, France, 27-29
  October 2010 - Volume 1}}. \bibinfo{publisher}{{IEEE} Computer Society},
  \bibinfo{address}{Arras}, \bibinfo{pages}{215--218}.
\newblock


\bibitem[\protect\citeauthoryear{Rago and Toni}{Rago and Toni}{2017}]%
        {Rago2017}
\bibfield{author}{\bibinfo{person}{Antonio Rago} {and}
  \bibinfo{person}{Francesca Toni}.} \bibinfo{year}{2017}\natexlab{}.
\newblock \showarticletitle{Quantitative Argumentation Debates with Votes for
  Opinion Polling}. In \bibinfo{booktitle}{\emph{{PRIMA} 2017: Principles and
  Practice of Multi-Agent Systems - 20th International Conference, Nice,
  France, October 30 - November 3, 2017, Proceedings}}.
  \bibinfo{publisher}{Springer}, \bibinfo{address}{Cham},
  \bibinfo{pages}{369--385}.
\newblock


\bibitem[\protect\citeauthoryear{Rago, Toni, Aurisicchio, and Baroni}{Rago
  et~al\mbox{.}}{2016}]%
        {RagoTAB16}
\bibfield{author}{\bibinfo{person}{Antonio Rago}, \bibinfo{person}{Francesca
  Toni}, \bibinfo{person}{Marco Aurisicchio}, {and} \bibinfo{person}{Pietro
  Baroni}.} \bibinfo{year}{2016}\natexlab{}.
\newblock \showarticletitle{Discontinuity-Free Decision Support with
  Quantitative Argumentation Debates}. In \bibinfo{booktitle}{\emph{Principles
  of Knowledge Representation and Reasoning: Proceedings of the Fifteenth
  International Conference, {KR} 2016, Cape Town, South Africa, April 25-29,
  2016.}} \bibinfo{publisher}{{AAAI} Press}, \bibinfo{address}{Cape Town, South
  Africa}, \bibinfo{pages}{63--73}.
\newblock


\bibitem[\protect\citeauthoryear{Toni}{Toni}{2014}]%
        {ABAtut}
\bibfield{author}{\bibinfo{person}{Francesca Toni}.}
  \bibinfo{year}{2014}\natexlab{}.
\newblock \showarticletitle{A tutorial on assumption-based argumentation}.
\newblock \bibinfo{journal}{\emph{Argument \& Computation}}
  \bibinfo{volume}{5}, \bibinfo{number}{1} (\bibinfo{year}{2014}),
  \bibinfo{pages}{89--117}.
\newblock


\end{thebibliography}


\section*{Appendix}

\begin{proof}[Proof of Theorem \ref{Theorem: Acceptability of an Assumption less than 2}]
    If $|\mathcal{A}| = 1$, the result holds vacuously as $\alpha \in \mathcal{A}$ must be accepted under all five semantics.
	If $|\mathcal{A}| = 2$ and $\alpha \in \mathcal{A}$ is accepted under those semantics, then $\alpha$ must be accepted in the aggregated framework as there cannot exist a deduction 
	$\{\beta\} \vdash^R \contrary{\alpha}$ for $\beta \in \mathcal{A} \setminus \{\alpha\}$ with $R \subseteq \mathcal{R}_{agg}$. By 
	definition of quota rules and oligarchic rules, then 
	no such $R$ can exist in the agents' 
	rule sets either. Thus, $\alpha$ is accepted 
	by all agents.
\end{proof}


\begin{proof}[Proof of Theorem \ref{Theorem: Preferred and Complete equal 4}]
    Let $P$ be  
	any property amongst being a preferred, complete, well-founded, or ideal extension. 
	To show that 
	$P$ is 
	not preserved by quota rules, we give counter examples. For $|\mathcal{A}| = 3$, assume four Bipolar ABA frameworks with two frameworks having rules $\mathcal{R}_{1,2} = \{\contrary{B} \leftarrow A, \quad \contrary{C} \leftarrow B\}$ and the other two frameworks having rules $\mathcal{R}_{3,4} = \{\contrary{A} \leftarrow B, \quad \contrary{B} \leftarrow C\}$. The set of assumptions $\Delta = \{A,C\}$ is preferred, complete, well-founded, and ideal in each framework. However, using the unanimity rule ($q = 4$) or oligarchic rule with veto powers given to all frameworks, 
	$\{A,B,C\}$ is  well-founded and ideal, as well as the only extension to be preferred, complete, in the aggregated framework, given that $\mathcal{R}_{agg} = \{\}$. With majority 
	or nomination rule, the extensions $\{A,C\}$ and $\{B\}$ are preferred, $\{A,C\}$, $\{B\}$, and $\{\}$ are complete, and 
	$\{\}$ is well-founded and ideal in the aggregated framework,  with 
	$\mathcal{R}_{agg} = \{\contrary{B} \leftarrow A, \quad \contrary{C} \leftarrow B, \quad \contrary{A} \leftarrow B, \quad \contrary{B} \leftarrow C\}$. Thus, $\Delta$ (and so $P$) is not preserved 
	using quota rules and oligarchic rules
	. 
    
	For $|\mathcal{A}| = 4$, assume three Bipolar ABA frameworks with rules $\mathcal{R}_1 = \{\contrary{A} \leftarrow D, \quad \contrary{D} \leftarrow B, \quad \contrary{C} \leftarrow D\}$, $\mathcal{R}_2 = \{\contrary{A} \leftarrow D, \quad \contrary{B} \leftarrow D, \quad \contrary{D} \leftarrow C\}$, and $\mathcal{R}_3 = \{D \leftarrow A\}$. 
	$\Delta = \{A,B,C\}$ is preferred, complete, well-founded, and ideal in each framework. However, using unanimity 
	($q = 3$) or oligarchic rules with veto powers given to all 
	agents, the  well-founded and ideal extension, as well as the only preferred, complete extension, is $\{A,B,C,D\}$ in the aggregated framework, as 
	$\mathcal{R}_{agg} = \{\}$. With majority 
	($q = 2$), 
	$\{B,C,D\}$ is preferred, complete, well-founded, and ideal in the aggregated framework, with 
	$\mathcal{R}_{agg} = \{\contrary{D} \leftarrow A\}$. Lastly, 
	nomination 
	gives $\{A,B,C\}$ and $\{D\}$ as the preferred extensions, $\{A,B,C\}$, $\{D\}$, and $\{\}$ as the complete extensions, and 
	$\{\}$ as the well-founded and ideal extension. Thus, $\Delta$ (and so $P$) is not preserved using quota 
	or
	oligarchic rules
	.
\end{proof}


\begin{proof}[Proof of Theorem \ref{Theorem: Preferred and Complete less than 2}]
    If $|\mathcal{A}| = 1$, the result holds vacuously as the only extension is $\Delta = \{\alpha\}$ for $\alpha \in \mathcal{A}$. 
    If $|\mathcal{A}| = 2$, if $\Delta = \{\alpha\}$ is preferred, complete, well-founded, or ideal in all the agents' frameworks, then, for each $i \in N$, there must be 
	a deduction 
	$\{\alpha\} \vdash^{R_i} \contrary{\beta}$
	. Thus, any quota 
	and oligarchic rules preserve $P$. The other cases ($\Delta = \{\beta\}$ and $\Delta = \{\alpha, \beta\}$) can be proven similarly.
\end{proof}


\begin{proof}[Proof of Theorem \ref{Theorem: Coherence 2 and 3}]
	Assume that coherence holds for each agent, for an extension $\Delta$. By contradiction, assume that $\Delta$ is not coherent in the aggregated $\mathcal{F}$. By unanimity
	, 
	$\mathcal{R}_i \supseteq \mathcal{R}_{agg}$ for all $i \in N$. Thus, $\Delta$ is not coherent in the agents' frameworks, 
	leading to a contradiction.
\end{proof}


\begin{proof}[Proof of Theorem \ref{Theorem: Coherence 1}]
	If $|\mathcal{A}|=1$ then $\mathcal{R}_i=\{\}$ (for all $i \in N$) and  
	$\{\alpha\}$ is coherent as it is closed, conflict-free, admissible, preferred, complete, set-stable, well-founded, and ideal. Therefore, using any quota or oligarchic rule, coherence also holds for the aggregated framework
	. 
\end{proof}

\end{document}